%% file: lowerbound.tex
\documentclass{article}

\usepackage{graphicx}
\usepackage{subfigure}
\usepackage{booktabs} 

\usepackage{hyperref}
\usepackage{macros}
\input{1-defs}

\usepackage[accepted]{icml2021}

\icmltitlerunning{Exponential Lower Bounds for Batch Reinforcement Learning}

\begin{document}

\twocolumn[
\icmltitle{Exponential Lower Bounds for Batch Reinforcement Learning: \\ Batch RL can be Exponentially Harder than Online RL}




\begin{icmlauthorlist}
\icmlauthor{Andrea Zanette}{to}
\end{icmlauthorlist}

\icmlaffiliation{to}{Institute for Computational and Mathematical Engineering, Stanford University, Stanford, USA}

\icmlcorrespondingauthor{Andrea Zanette}{zanette@stanford.edu}

\icmlkeywords{Machine Learning, ICML}

\vskip 0.3in
]



\printAffiliationsAndNotice{}  

\begin{abstract}
\input{2a-abstract}
\end{abstract}
\input{2b-body}

\bibliography{rl}
\bibliographystyle{icml2021}

\newpage
\onecolumn
\appendix
\input{3-appendix}


\end{document}

%% file: 2a-abstract.tex
Several practical applications of reinforcement learning involve an agent learning from past data without the possibility of further exploration. Often these applications require us to 1) identify a near optimal policy or to 2) estimate the value of a target policy. For both tasks we derive \emph{exponential} information-theoretic lower bounds in discounted infinite horizon MDPs with a linear function representation for the action value function even if 1) \emph{realizability} holds, 2) the batch algorithm observes the exact reward and transition \emph{functions}, and 3) the batch algorithm is given the \emph{best} a priori data distribution for the problem class. Our work introduces a new `oracle + batch algorithm' framework to prove lower bounds that hold for every  distribution. The work shows an exponential separation between batch and online reinforcement learning.

%% file: 2b-body.tex
\section{Introduction}
While the grand goal of reinforcement learning (RL) is to design fully autonomous agents capable of improving their performance over time by learning from past mistakes, oftentimes a \emph{batch} approach --- which predicts some quantity of interest using past observations only --- is preferable. For example, past data may be available in large quantities and should not be disregarded by adopting a purely online method. In other applications, safety concerns require that the dataset be collected by a carefully monitored procedure, involving a human or a safe controller. In addition, batch algorithms are key tools to build  complex online procedures. 

These considerations motivate us to investigate whether there exists any \emph{fundamental limitation} to using a batch approach for RL compared to online learning. Concretely, we consider two classical batch RL problems: 1) the \emph{off-policy evaluation} (OPE) problem, where the batch algorithm needs to predict the performance of a \emph{target policy} and 2) the \emph{best policy identification} (BPI) problem, where the batch algorithm needs to identify a near optimal policy. 

\paragraph{Linear function approximation} 
As many applications of RL require very large state and action spaces, the hope is that by leveraging strong prior knowledge, for example on the form of the optimal solution, we can learn the critical features of a Markov decision process (MDP) to solve the task at hand without probing the full state-action space. 
For the OPE problem, we assume that the action-value function $\Qpi$ of the target policy $\pi$ can be written at any state action pair $(s,a)$ as an inner product $\phi(s,a)^\top \theta$ between a known feature extractor $\phi$ and an unknown parameter $\theta$ that the learner seeks to identify. For the BPI problem, this representation condition applies to the action-value function $\Qstar$ of an optimal policy.

\paragraph{Quality of the dataset and assumptions}
Batch algorithms are limited by the quality and quantity of the available data: for example, in a tabular MDP, without adequate coverage of the area of the MDP that the optimal policy tends to visit, there is little hope that we can identify such policy or even predict its performance. However, given enough samples in each state-action pairs,  it becomes easy to identify the optimal policy \citep{Azar12} or to evaluate the value of another \citep{yin2020near}. In addition, without any prior knowledge about the problem or information about the target policy, a uniform distribution over the state-action space is a priori optimal. 

Very recently, some authors have derived exponential lower bounds for the off-policy evaluation problem with linear function approximations \citep{wang2020statistical,amortila2020variant}. They discover that even if a sampling distribution induces the best-conditioned covariance matrix, at least exponentially many samples are needed to estimate the value of a target policy reasonably well. As their hard instances are tabular MDPs (i.e., with small state-action spaces), certainly \emph{there exists a better batch distribution} for their setting: the OPE problem in tabular MDPs is easily solvable by using a uniform distribution over the state-actions coupled with the policy evaluation algorithm on the empirical model. Therefore, their works show that the assumption on the condition number of the covariance matrix is insufficient in batch RL. However, their works leave open the question of whether the OPE problem with linear value functions can \emph{always} be solved by using a better sampling distribution that generates a dataset of polynomial size (in the horizon and feature dimension), leading to the following, more fundamental  question in the context of batch RL with function approximation:
\begin{center}
\textbf{What can a batch algorithm learn using the best a priori distribution for the problem class at hand?}
\end{center}
In particular, we wonder whether  there is any penalty in using an entirely batch approach in place of an online (and adaptive) method if a good dataset is provided to the batch algorithm. To answer this question for both the OPE and BPI problems in the strongest possible way, we consider an auxiliary process or \emph{oracle} that provides the \emph{batch algorithm} with the best data distribution for the task; together they form a \emph{learning algorithm}. However, in contrast to fully online / adaptive algorithms, the oracle is not allowed to change its data acquisition strategy while information is being acquired. As we explain in \cref{lowerbound:sec:Benfits}, this framework allows us to derive strong batch RL lower bounds, because they \emph{will hold for every batch distribution}. Thanks to this framework, we can recover the concurrent lower bound by \cite{wang2020statistical} --- which holds for one specific choice of the sampling distribution --- in infinite horizon RL as a corollary.

\paragraph{Learning process and assumptions} 
We consider oracles that can decide on a set of strategic policies to gather data. 
Alternatively, the oracles can directly specify the state-actions where they want to observe the rewards and transitions (for example, with a simulator). In either case, the oracle selects a sampling strategy and the batch algorithm later receives a \emph{dataset} of states, actions, rewards and transitions. Using the dataset, the batch algorithm makes a prediction, i.e., it  estimates the value of a target policy (OPE problem) or returns a near optimal policy (BPI problem). 
We make two other assumptions that favor the learner:
\begin{itemize}
\item \emph{Realizability}, 
i.e., there is no misspecification.
\item \emph{Exact feedback}: the exact reward and transition \emph{function} is observed for each point in the dataset. This is equivalent to observing \emph{infinite data}.
\end{itemize} 
\subsection{Contributions}
With such strong assumptions and the freedom of selecting any state-action, identifying a near optimal policy or predicting the value of another should be easy tasks for the learner. For example, in tabular MDPs the oracle could prescribe one query in each state-action pair. This way, the reward and transition functions would become known everywhere, giving  the batch algorithm  complete knowledge of the MDP. For linear bandits with feature vectors spanning $\R^d$, we know that $d$ queries along a basis suffice to identify the full bandit instance. For $H$ horizon MDPs with linear representations, backward induction with $dH$ queries and exact feedback precisely identifies the MDP at hand. 
In all these cases the oracle can find a good sampling distribution for the batch learner, and there is little or no advantage to using an  oracle that adapts the query selection strategy as feedback is received. However, in infinite horizon problems the situation is quite different; we show that
\begin{enumerate}
\item there exists OPE and BPI problems where any batch algorithm must receive an \emph{exponential} $\approx (\frac{1}{1-\gamma})^d$ dataset to return a good answer, 
\item if the dataset does not originate from policy rollouts then the lower bounds hold even if the action-value function of \emph{every policy} admits a linear representation,
\item there exist exponentially hard batch BPI problems (even under the best data distribution) which are easy to solve with online / adaptive algorithms, showing \emph{an exponential separation between batch and online RL},
\item there exist exponentially hard problems for infinite horizon batch RL which cannot arise in finite horizon problems, showing \emph{exponential separation between finite and infinite horizon batch RL}.
\end{enumerate}
As a corollary, our work recovers \cite{wang2020statistical}'s lower bound in infinite horizon RL and also shows that the classical globally optimal experimental design yields \emph{provably bad} distributions for infinite horizon RL, making learning impossible even in the limit of infinite data; we discuss this in \cref{lowerbound:sec:Benfits}.

We make the following \emph{technical} contributions: 
\begin{enumerate}
\item we introduce a new `oracle + batch algorithm' framework to derive lower bounds for \emph{every a priori distribution}; this automatically yields fixed distribution lower bounds (e.g., \citep{wang2020statistical}) as a special case,
\item we help formalize the hardness induced by the \emph{deadly triad} \citep{sutton2018reinforcement}, i.e., the combination of off-policy learning, bootstrapping and function approximation; in particular, we explain that the  \emph{bootstrapping} problem is fundamentally different and potentially more severe than the \emph{extrapolation} problem,
\item we present new classes of hard MDPs where critical MDP information is `deferred’ --- through bootstrapping --- and hidden in an unknown and exponentially small region \emph{in feature space}, too small to be covered by a batch dataset but easy to locate using an online algorithm. 
\end{enumerate}

\subsection{Literature}
Polynomial lower bounds are often derived to certify that a certain algorithm is sample efficient \citep{jiang2016doubly,duan2020minimax,hao2020sparse}; in this work we are interested in exponential lower bounds. Divergence of dynamic programing algorithms with function approximation is well known \citep{baird1995residual,tsitsiklis1996feature} and has prompted researchers to look for information-theoretic lower bounds for generic predictors \citep{chen2019information} and in presence of misspecification \citep{du2019good}. Concurrently,  \citet{weisz2020exponential,wang2020statistical} also show information-theoretic lower bounds highlighting the danger of \emph{extrapolation}; 
for additional literature, please see app. \ref{app:Literature}.

\section{Preliminaries}
\label{sec:Defs}
Discounted infinite horizon MDPs \citep{puterman1994markov} are defined by a tuple $M = \langle \StateSpace,\ActionSpace,p,r,\gamma \rangle$, where
$\StateSpace$ is the state space, $\ActionSpace_s$ is the action space in state $s\in\StateSpace$ and $\ActionSpace = \cup_{s\in \StateSpace}\{\ActionSpace_s\} $ is the set of all state-dependent action spaces. We assume that for each state $s \in \StateSpace$ and action $a \in \ActionSpace_s$ there exists a measure $p (\cdot \mid s, a)$ representing the transition dynamics and a scalar reward function $r(s,a) \in [-1,1]$. 
The discount factor $\gamma$ is assumed to be in $(0,1)$. We consider the discounted return of any policy to be $\in [-1,+1]$  (without loss of generality upon rescaling the reward function). 
A deterministic policy $\pi$ maps every state $s\in \StateSpace$ to an action $a\in\ActionSpace_s$. 
The value function of policy $\pi$ in state $s \in \StateSpace$ is defined as $V^\pi(s)= \sum_{t=0}^\infty \gamma^{t}\E_{x_t\sim \pi \mid x_0 = s}r(x_t,\pi(x_t))$;  the limit always exists and it is finite under our assumptions as long as the expectation is well defined. The action-value function $Q^\pi$ of policy $\pi$ in $(s,a)$ also exists and is defined as $Q^\pi(s,a)= r(s,a) + \sum_{t=1}^\infty \gamma^{t}\E_{x_t\sim \pi \mid (s,a)}r(x_t,\pi(x_t))$.
An optimal policy $\pistar$, when it exists, is defined in every state as $\pistar(s) = \argmax_\pi V^\pi(s)$ and the corresponding value function and action-value function are denoted  with $\Vstar = V^{\pistar}$ and $\Qstar = Q^{\pistar}$. We sometime add the subscript $M$ to indicate that a certain quantity depends on the MDP $M$ under consideration (e.g., we write $\Vstar_M$). The Euclidean ball in $\R^d$ is defined as $\B = \{ x\in\R^d \mid \| x \|_2 \leq 1 \}$.
The \emph{Bellman optimality operator} $\T$ and the \emph{Bellman evaluation operator} $\T^\pi$ are mappings between action value functions: 
\begin{align*}
(\T Q)(s,a) & = r(s,a) + \gamma\E_{s' \sim p(s,a)} \sup_{a' \in \ActionSpace_{s'}}Q(s',a') \\ 
(\T^{\pi} Q)(s,a) & = r(s,a) + \gamma\E_{s' \sim p(s,a)} Q(s',\pi(s')).
\end{align*}
\section{Batch Reinforcement Learning}
We first formally define the two  \emph{learning problems}, namely the task of returning a near optimal policy, also known as \emph{best policy identification} (BPI) problem, and the task of predicting the value of a \emph{target policy}, also known as \emph{off-policy evaluation} (OPE) problem. Then, in the next four sub-sections we describe the learning process in more detail. The lower bounds that we later derive will hold for all algorithms of the form described in this section.

A \emph{BPI problem} $(s^\star,\{ M \in \M\} )$ is defined by a starting state $s^\star$ and a class $\M$ of MDPs sharing the same state space, action space and discount factor. 
An \emph{OPE problem} $(s^\star,\{(M,\pi_M) \mid M \in \M\})$ additionally requires us to identify one or more target policies
$\pi_M$ for each $M \in \M$. 

Depending on the problem, the oracle selects a sampling strategy such that for every \emph{problem instance} $(s^\star,M,\pi_M)$ (of an OPE problem) or $(s^\star,M)$ (of a BPI problem) the batch algorithm  experiences a dataset $\mathcal D$ from $M$ and uses it to return an accurate estimate $\Qhat_{\mathcal D}$ of the action-value function of $\pi_M$ at $s^\star$ (for the OPE problem) or a near optimal policy $ \pihatstar_{\mathcal D}$ on $M$ from $s^\star$ (for the BPI problem).
\begin{algorithm}[H]
\footnotesize
\floatname{algorithm}{Protocol}
\begin{algorithmic}[1]
\caption{\footnotesize \textsc{Batch RL with a Strategic Oracle}}
\label{alg:Policy:BatchLearner}
\STATE \emph{(Input)} Oracle receives a batch problem and budget $n$
\STATE \emph{(Query Selection)} Oracle chooses $\mu$ (def. \ref{lowerbound:asm:PolicyFree})
or $T$ (def. \ref{lowerbound:asm:PolicyInduced}).
\STATE \emph{(Data Collection)} Batch algorithm receives the dataset $\mathcal D$
\STATE \emph{(Output)} Batch algorithm returns $\Qhat_{\mathcal D}$ or $\pihatstar_{\mathcal D}$ 
\normalsize
\end{algorithmic}
\end{algorithm}
\subsection{Step I: Input}
\label{sec:step1}
The oracle receives \emph{either} an OPE or a BPI problem together with a \emph{query budget} $n\in \mathbb N$. The oracle has access to each instance of the problem it is given, and in particular it knows the state and action spaces $\StateSpace,\ActionSpace$, the discount factor $\gamma$, and the full set of transition functions and reward functions for each MDP $M\in\M$, but it does not know which of these MDPs in $\M$ it will interact with or what the target policy will be. In addition, the oracle has full knowledge of the batch algorithm.

\subsection{Step II: Query Selection}
\label{sec:step2} 
The purpose of the oracle is to help the batch algorithm by providing it with the best dataset for the specific problem at hand. 
To capture different mechanisms of data acquisition, we consider two methods to specify the \emph{query set}, i.e., the set of state-action pairs where the oracle wants the batch algorithm to observe the rewards and the transitions. In the first mechanism the oracle directly selects the state-actions; we place no restriction on the mechanism to obtain these queries as long as the sampling distribution is fixed for all MDPs in the class.
\begin{definition}[Policy-Free Queries]
\label{lowerbound:asm:PolicyFree} 
A set $\mu$ of state-action pairs is said to be policy-free for an OPE or BPI problem if $\mu$ does not depend on the specific MDP instance $M \in \M$.
\end{definition}
Since the batch algorithm observes the exact reward and transition function at the selected query points, the number of policy-free queries is \emph{the size of the support of the distribution} $\mu$.

The second mechanism to collect data is by selecting deterministic\footnote{This is not a restriction: for any given stochastic policy the agent can sample an action from the distribution of actions in every state before deploying the policy.} policies. The policies will generate random trajectories that the batch algorithm later observes. Since we are interested in what the agent can learn in the limit of infinite data (i.e., under exact feedback) we let the batch algorithm observe \emph{all possible realizations} of such trajectories. With this aim, we define the state-action space reachable in $c$ or less timesteps from $s_0$ using policy $\pi$ as $\Reach(s_0,\pi,c) = $
\begin{align*}
\{(s,a) \mid \exists t < c, \; \text{s.t.} \; \Pro\( (s_t,a_t) = (s,a) \mid \pi,s_0 \) > 0  \}
\end{align*}  
where $(s_t,a_t)$ is the random  state-action encountered at timestep $t$ upon following $\pi$ from $s_0$.
We let the oracle select the best combination of trajectory lengths and number of distinct policies.

\begin{definition}[Policy-Induced Queries]
\label{lowerbound:asm:PolicyInduced}
Consider an OPE or BPI problem and fix a set $T = \{ (s_{0i},\pi_i,c_i) \}_{i=1}^\kappa$ of triplets, each containing a starting state $s_{0i}$, a deterministic policy $\pi_i$ and a trajectory length $c_i$ such that $\sum_{i=1}^\kappa c_i \leq n$. 
Then the query set $\mu$ induced by $T$ is defined as
\begin{align*}
\mu \defeq \bigcup_{(s_{0},\pi,c) \in T} \Reach(s_{0},\pi,c).
\end{align*}
\end{definition} 
The condition $\sum_{i=1}^\kappa c_i \leq n$ attempts to make the amount of information acquired using policy-free queries comparable to the policy-induced query method, although the latter generates $|\mu| \geq n$ if the dynamics are stochastic. In other words, the batch learner always observes at least $n$ state-actions if these are induced by policies.

\begin{remark}
\label{rem:PolicyInducedIsPolicyFree}
A policy-induced query set is also policy-free whenever the dynamics of each MDP $M\in\M$ are the same. 
\end{remark} 

\begin{remark}
	When prescribing a set $T$ the oracle has knowledge of the induced dataset $\mu$ for each choice of $T$ (e.g., by having access to a connectivity graph of the MDP).
\end{remark}

While a policy-free query set may seem less constrained than a policy-induced query set, the latter may implicitly reveal additional information about the dynamics of the MDP whenever the induced trajectories are all different across different MDPs.

\subsection{Step III: Data Collection}
After the oracle has submitted the sampling strategy, the batch learner receives a dataset $\mathcal D$ which contains the \emph{exact} reward and transition \emph{function}, $r(s,a)$ and $p(s,a)$, from the MDP $M \in \M$ for each $(s,a)$ in the query set $\mu$. 
\subsection{Step IV: Output}
\label{sec:step4}
The batch algorithm is finally required to make a prediction using the acquired dataset $\mathcal D$. For the OPE problem, the batch algorithm also receives the target policy $\pi_M$ and it is required to output an estimator $\Qhat_{\mathcal D}$ for the action-value function of the target policy $Q_M^{\pi_M}(s^\star,\cdot)$; for the BPI problem, it must output a near-optimal policy $\pihatstar_{\mathcal D}$ at $s^\star$.

\subsection{Evaluation Criterion}
\label{sec:evaluation}
The \emph{oracle} and the \emph{batch algorithm} together form a \emph{learning algorithm}. This framework allows us to derive batch lower bounds in a strong form as they will hold for any data distribution. We say that the learning algorithm is \emph{$(\epsilon,\delta)$-sound} for an OPE problem
if for every instance 
$(s^\star,M,\pi_M)$ of the problem the returned estimator $\widehat Q_{\mathcal D}$ is accurate w.h.p.:
$$\Pro\( \sup_{a\in \ActionSpace_{s^\star}}|(Q_M^{\pi_M} - \Qhat_{\mathcal D})(s^\star,\cdot) | < \epsilon \) > 1-\delta. $$
Similarly, we say that that the learning algorithm is \emph{$(\epsilon,\delta)$-sound} for a BPI problem if for every instance $(s^\star,M)$  it holds that the returned policy $\pihatstar_{\mathcal D}$ is near optimal w.h.p.:
$$\Pro\( (\Vstar_M - V_M^{\pihatstar_{\mathcal D}})(s^\star)  < \epsilon \) > 1-\delta.$$
As the query set is always non-random and the batch algorithm experiences the exact reward and transition function, the only  randomness lies in the possible randomization internal to the batch algorithm when it returns an answer.

\subsection{Adaptive and Online Algorithms}
Consider a policy-free mechanism. We say that a learning algorithm is \emph{adaptive} if every time the oracle  submits a state-action it receives the feedback from the environment and can use it to select the next state-action to query the MDP.

Likewise, consider a policy-induced mechanism. We say that a learning algorithm is acting \emph{online} if every time the oracle selects a policy and trajectory length, it can use the acquired feedback to  select the next combination of policy and trajectory length (its position is reset in every episode).
\section{Intuition}
\label{main:sec:Intuition}
\begin{figure*}[t!]
\centering    
\begin{adjustbox}{minipage=\linewidth,scale=1.0}
\centering
\subfigure{
\input{fig1.tex}}
\subfigure{
\input{fig2.tex}}
\hspace{-2mm}
\subfigure{
\input{fig3.tex}}
\hspace{-0mm}
\caption{ 
\emph{Left}: the orthogonal blue vectors $\phi_1,\phi_2$ represent the choice of the algorithm, while $\phi^+_1,\phi^+_2$ are the corresponding successor features. The learner thus acquires information only along the vertical direction.
\emph{Middle}: in the figure $\mathcal R = \Range(\Phi^\top-\gamma(\Phi^+)^\top)$ and $\mathcal N = \Null(\Phi-\gamma\Phi^+)$. 
\emph{Right}: in this case the next-state discounted feature $\gamma\phi^+_i$ cannot project $\phi_i$ onto the vertical plane because of the discount factor $\gamma < 1$.}
\label{lowerbound:fig:main}
\end{adjustbox}
\end{figure*}
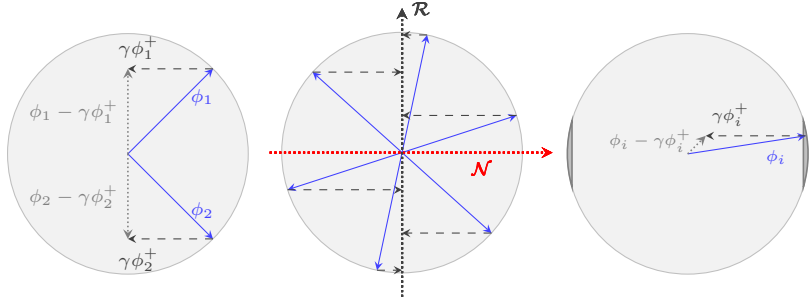
The mechanism that induces hardness for infinite horizon problems must be different than that in finite horizon: the constructions from \citet{weisz2020exponential,wang2020provably} rely on the \emph{extrapolation} issue that compounds the errors multiplicatively. In our case, the reward and transition functions are observed exactly, so there is no error to extrapolate in the first place. Instead, \emph{bootstrapping}, i.e., the fact that the value function in one state depends on the same value function at successor states \citep{sutton2018reinforcement}, is the root cause of hardness; here we provide some intuition. 

While each of our theorems need a different construction, they all build on the intuition presented in this section; at a high level, bootstrapping can ``erase'' the information gained along certain directions in feature space.


Suppose that the oracle is trying to find a good policy-free query set for the off-policy problem on a class of MDPs with feature vectors $\phi(\cdot,\cdot)$ anywhere in the unit Euclidean ball $\B$. 
Denote with $(s_1,a_1),\dots,(s_n,a_n)$ the state-actions chosen by the oracle and with $(s^+_1,a^+_1),\dots,(s^+_n,a^+_n)$ the corresponding successor states and actions\footnote{Assume deterministic successor states; the actions in the successor states are determined by the target policy (OPE problem) or by an $\argmax$ function (BPI problem).}.
Intuitively, choosing a set $\phi(s_1,a_1), \dots, \phi(s_n,a_n)$ of orthogonal feature vectors of maximal length seems to be the best the oracle can do because it gives rise to a covariance matrix $\Phi^\top\Phi$ (where $\Phi$ is described below) that is a multiple of the identity. In \cref{lowerbound:fig:main} (left) we represent the feature space of an RL problem where the learner can choose any feature vector in the Euclidean ball in $\R^2$; the agent's choice of the feature vectors $\phi_i=\phi(s_i,a_i)$ arise from globally optimal experimental design \citep{pukelsheim2006optimal}. 
Define
\begin{align*}
\hspace{-2mm}
\Phi = \begin{bmatrix}
\phi(s_1,a_1)^\top \\
\dots \\
\phi(s_n,a_n)^\top
\end{bmatrix},
r = \begin{bmatrix}
r(s_1,a_1) \\
\dots \\
r(s_n,a_n) 
\end{bmatrix},  
\Phi^+ = \begin{bmatrix}
\phi(s^+_1,a^+_1)^\top \\
\dots \\
\phi(s^+_n,a^+_n)^\top
\end{bmatrix}
\end{align*}

After observing the reward and successor states, we can write down the local (i.e., at the state-actions chosen by the oracle) Bellman equations. If we leverage the known functional form of the action-value function, we can write the following linear system, whose solution (in terms of $\theta \in \R^d$ or action-value function $\Phi\theta$) the batch algorithm seeks to discover:
\begin{align}
\label{eqn:FinalLinearSystem}
\Phi\theta = r + \gamma\Phi^+\theta \quad \longrightarrow \quad  (\Phi -\gamma\Phi^+)\theta = r.
\end{align}
Unlike machine learning or bandit problems where the agent directly observes a response from $\phi(s_i,a_i)$,  bootstrapping the future returns makes the agent observe a response (i.e., the reward) corresponding not to $\phi(s_i,a_i)$ but to  $\phi(s_i,a_i) - \gamma \phi(s^+_i,a^+_i)$, which we call ``effective feature vector'', see \cref{eqn:FinalLinearSystem}. 
Unfortunately, $\phi(s^+_i,a^+_i)$ can act adversarially. As an example, consider \cref{lowerbound:fig:main} (left) in $\R^2$ where the orthogonal feature vectors chosen by the agent are projected along the vertical direction by bootstrapping: the net effect is that the agent only learns along one axis. Mathematically, when
this happens the final system of linear equations in \cref{eqn:FinalLinearSystem} admits infinitely-many solutions since $\Phi - \gamma \Phi^+$ is rank 
deficient\footnote{This follows from the fundamental theorem of linear algebra. In particular, since the row space of $\Phi - \gamma\Phi^+$ does not span $\R^d$, the nullspace of $\Phi - \gamma\Phi^+$ must have dimension at least $1$.}. From a reinforcement learning perspective, the local Bellman equations in \cref{eqn:FinalLinearSystem} admit multiple linear value functions as possible solutions. This happens because the rewards and transitions that the batch algorithm has received could have originated from any MDP whose action-value function is a solution to the local Bellman equations.

Unfortunately, choosing more feature vectors does not easily solve the problem (\cref{lowerbound:fig:main} (middle)). To acquire information in all directions (i.e., to ensure that $\Phi - \gamma\Phi^+$ in \cref{eqn:FinalLinearSystem} is full rank)  the oracle must probe the small spherical cap in \cref{lowerbound:fig:main} (right). Which spherical cap to probe is unknown ahead of sampling (as the target policy or the optimal policy are unknown), and thus the oracle needs to probe all of them if it wants the batch algorithm to confidently solve the problem. Using the fact that there are exponentially many spherical caps in an Euclidean ball in $\R^d$, we conclude.
\paragraph{Exponential separation with online learning}
An online algorithm can detect that the next-state feature matrix $\Phi^+$ is acting adversarially. In addition, $\Phi^+$ reveals the location of the spherical cap in \cref{lowerbound:fig:main} (right). The online algorithm can then probe the spherical cap to ensure that the linear system in \cref{eqn:FinalLinearSystem} is full rank.

\section{Exponential Lower Bounds}
\label{sec:MainResult}
We define the \emph{query complexity to $(\epsilon,\delta)$-soundness of a problem} (OPE or BPI) to be the minimum value for $n$ (as in \cref{lowerbound:asm:PolicyFree,lowerbound:asm:PolicyInduced}) such that there exists a $(\epsilon,\delta)$-sound learning algorithm for that problem. In particular, the query complexity depends on the MDP class $\M$ and can be different for the OPE and BPI tasks and for the policy-free and policy-induced query mechanism. Our lower bounds on the query complexity are \emph{significantly stronger} than typical sample complexity lower bounds: they are really lower bounds on the \emph{size of the support} of the distribution $\mu$  and automatically imply   \emph{infinite sample complexity} lower bounds  since the batch algorithm already observes the exact reward and transition functions where $\mu$ is supported.

The lower bounds 
are expressed in terms of the regularized incomplete beta function $I_x(a,b) = B(x,a,b)/B(1,a,b)$ where $B(x,a,b)$ is the incomplete beta function for some positive real numbers $a,b$ and $x \in [0,1]$; for the precise definitions, please see \cref{app:additional_notation}. For brevity, define $
	\mathcal N(\gamma,d) = I^{-1}_{1-\gamma^2}\(\frac{d-1}{2},\frac{1}{2}\)
$. \cref{cor:LowerBound} ensures $\mathcal N(\gamma,d)$ is exponential in the dimension $d$ for $d \geq 5$ (here the $\approx$ symbol highlights an approximate dependence without a formal definition):
\begin{align*}
\mathcal N(\gamma,d) > 2^{-d}\mathcal N(\gamma,d) & \geq \gamma  \sqrt{d} \( \frac{1}{2^{2.5}(1-\gamma)} \)^{\frac{d-1}{2}} \\
& \approx \(\frac{1}{1-\gamma}\)^d.
\end{align*}
\subsection{Realizability Assumptions}
Unless additional assumptions are made regarding the MDP class $\M$ that defines the OPE and BPI problems, the query complexity of a reasonably sound  learner is exactly the size of the state and action space (we assume exact feedback). One hopes that by restricting the MDP class $\M$, the query complexity of the OPE and BPI problems can be brought down to a more manageable level, in particular, independent of the state-action spaces.

We make one of the following three assumptions (only one assumption will hold at any given time, depending on the theorem); two are known as realizability, and the third is strictly stronger than the first two. The first concerns the OPE problem; $\B$ is the unit Euclidean ball.

\begin{assumption}[$Q^\pi$ is Realizable]
\label{lowerbound:asm:OPE:MDP}
Given an OPE problem $(s^\star,\{(M,\pi_M), M\in\M\})$, there exists a $d$-dimensional map $\phi(\cdot,\cdot)$ such that $\|\phi(\cdot,\cdot)\|_2 \leq 1$ and for any  $M\in \M$ the action-value function of the target policy $\pi_M$ satisfies $\forall (s,a), Q_M^{\pi_M}(s,a) = \phi(s,a)^\top \theta_M$ for some $\theta_M \in \B$.
\end{assumption}
For the BPI problem the representation condition applies to the action-value function  of an optimal policy.

\begin{assumption}[$\Qstar$ is Realizable]
\label{lowerbound:asm:BPI:MDP}
Given a BPI problem $(s^\star,\{M\in\M\})$, there exists a $d$-dimensional feature map $\phi(\cdot,\cdot)$ such that $\|\phi(\cdot,\cdot)\|_2 \leq 1$ and for any  $M\in \M$ there exists $\theta_M^\star\in \B$ such that the optimal action-value function is linear: $\forall (s,a), \Qstar_M(s,a) = \phi(s,a)^\top \theta_M^\star$.
\end{assumption}

A stronger assumption we make is that the action-value function of \emph{every} policy has a linear representation.

\begin{assumption}[$Q^\pi$ is Realizable for every Policy]
\label{lowerbound:asm:ALL:MDP}
Given a BPI problem $(s^\star,\{M\in\M\})$ or an OPE problem $(s^\star,\{(M,\pi_M), M\in\M\})$, there exists a $d$-dimensional feature map $\phi(\cdot,\cdot)$ such that $\|\phi(\cdot,\cdot)\|_2 \leq 1$ and for any MDP $M\in\M$ the value of every policy $\pi$ satisfies $\forall (s,a), Q_M^{\pi}(s,a) = \phi(s,a)^\top \theta^{\pi}_M$ for some $\theta^{\pi}_M \in \B$.
\end{assumption}

The learners that we consider are aware of these assumptions because they can examine each MDP in the class $\M$ they receive (see \cref{sec:step1}).

\subsection{Off-Policy Evaluation} 
 
The first result of this work is contained in the following lower bound for the OPE problem; since all MDPs in $\M$ share the same dynamics, the oracle has full knowledge of the actions it needs to take to visit any state-action it desires.  
\begin{theorem}[OPE Policy-Induced Lower Bound]
\label{lowerbound:thm:OPE:Policy-Induced}
There exists an OPE problem $(s^\star,\{(M,\pi_M), M\in\M\})$  satisfying \cref{lowerbound:asm:OPE:MDP} such that its policy-induced query complexity to $(1,1/2)$-soundness  is at least $
\mathcal N(\gamma,d)$. 
\end{theorem}

It is useful to compare the \emph{form} of the above lower bound with that of some concurrent results for the off-policy evaluation problem: for example, \citep{wang2020statistical,amortila2020variant}  show that \emph{there exist} a sampling distribution $\mu$ (that induces the best-conditioned covariance matrix), a target policy $\pi$ and an MDP class $\widetilde \M$, each satisfying certain properties, such that the best estimator on the most difficult MDP in $\widetilde \M$ performs poorly if less than exponentially many  samples are used. However, in their case there exists a better batch distribution of polynomial size that solves the problem. 
Our instances are instead much harder, as they command an exponential dataset even for the \emph{best} distribution for the task: since the oracle can prescribe any data distribution, we can claim that \emph{for all} distributions of polynomial size we can find an MDP subclass $\widetilde \M \subseteq \M$ and a target policy such that all MDPs in $\widetilde \M$ generate similar datasets but the value of the target policy is very different on these MDPs in $\widetilde \M$, giving lower bounds in a \emph{stronger form}.
 
\subsection{Best Policy Identification}
As in \cref{lowerbound:thm:OPE:Policy-Induced}, in the next lower bound the oracle knows the set $\mu$ induced by any choice of $T$ (see \cref{lowerbound:asm:PolicyInduced}).
\begin{theorem}[BPI Policy-Induced Lower Bound]
\label{lowerbound:thm:BPI:Policy-Induced}
There exists a BPI problem $(s^\star,\{M\in\M\})$ satisfying \cref{lowerbound:asm:BPI:MDP} with features in dimension $d+1$ such that its policy-induced query complexity to $(1/2,1/2)$-soundness is at least $ 
2^{-d}\mathcal N(\gamma,d)$.
\end{theorem} 
Hard BPI problems for adaptive and online algorithms are given by \citet{weisz2020exponential}. Clearly, their construction can be embedded in an infinite horizon MDP, giving a lower bound under their assumptions. However, our framework allows the learner to observe the \emph{exact} reward and transition function, which is equivalent to having infinite data at the selected state-actions. In such case, the construction of \citet{weisz2020exponential} no longer gives rise to hard instances. In particular, the BPI problem with our assumptions becomes straightforward in finite horizon, showing an exponential separation between finite and infinite horizon batch RL.

\subsection{Lower Bounds with Stronger Representation}

We wonder what is achievable if the oracle can specify the queries anywhere in the state-action space without the restriction imposed by following policies (i.e., using a policy-free query set). Under this assumption the situation gets surprisingly worse, as the lower bounds now hold even if the action-value function of \emph{every} policy is linear. 

\begin{theorem}[Policy-Free Lower Bounds]
\label{lowerbound:thm:ALL:Policy-Free}
There exist an OPE problem $(s^\star,\{(M,\pi_M), M\in\M\})$ and a BPI problem $(s^\star,\{M\in\M\})$, which satisfy \cref{lowerbound:asm:ALL:MDP} and share the same $s^\star$ and $\M$, such that their policy-free query complexity to $(1,1/2)$-soundness is at least $
\mathcal N(\gamma,d)$. In addition, an MDP class $\M$ that yields the lower bounds (but with $\frac{1}{\sqrt{d}},1/2)$-soundness) can be constucted with at most $|\ActionSpace| = 2d$ actions. 
\end{theorem}

Contrasting \cref{lowerbound:thm:ALL:Policy-Free} with \cref{lowerbound:thm:OPE:Policy-Induced,lowerbound:thm:BPI:Policy-Induced} shows that there is a sharp distinction in what can be achieved depending on the \emph{assumptions on the mechanism} that generates the batch dataset, a distinction which is absent in tabular RL.

\section{Exponential Separation with Online Learning}
\begin{theorem}[Exponential Separation with Online Learning]
\label{lowerbound:thm:BPI:ExponentialSeparation}
Consider the same BPI problem as in \cref{lowerbound:thm:BPI:Policy-Induced}. There exists an  online algorithm that can identify an optimal policy with probability one by observing trajectories of length one with exact feedback from $d+1$ distinct policies from $s^\star$.
\end{theorem} 
 This result shows \emph{exponential separation with online learning} even when the best batch distribution is used. The key information is hidden in an exponentially small area of the feature space whose position is a priori unknown. This region is too small to be covered by a batch dataset. However, an online algorithm can learn where this information is hidden and then probe such region as we show in \cref{main:sec:Proof}.
 
A similar exponential separation with online learning is available for the BPI problem in \cref{lowerbound:thm:ALL:Policy-Free} (see the end of the proof in the appendix). In addition, assuming access to a generative model an even stronger result is readily available in the literature using the Least-Square Policy Iteration (LSPI) algorithm \citep{lagoudakis2003least}: with a generative model \cite{lattimore2020learning} show that with probability at least $1-\delta$, LSPI
finds an $\epsilon$ optimal policy \emph{for any BPI problem} $(s^\star,\{M \in \M\})$ that satisfies \cref{lowerbound:asm:ALL:MDP} using at most $\poly(d,\frac{1}{1-\gamma},\frac{1}{\epsilon},\ln \frac{1}{\delta})$ samples. However, LSPI is non-batch as it relies on Monte-Carlo rollouts at every iteration. Our result thus shows that it is not possible to start from a batch dataset obtained from a generative model and run LSPI (or any algorithm) successfully without acquiring further data even if a  strong representation holds. 

\section{Proof Sketch of Theorem \ref{lowerbound:thm:BPI:Policy-Induced} and Theorem \ref{lowerbound:thm:BPI:ExponentialSeparation}}
\label{main:sec:Proof}
We first describe the state and action space which are fixed across all MDPs in the class $\M$. Then we describe the instance-dependent reward and transition functions. Finally we prove the theorems.

At a high level, each MDP contains a two-armed bandit instance in $s^\star$ (the starting state). There, the learner has two choices: 1) take the special action $a^\star$ that gives a known return 
or 2) take any other action in the positive orthant $\mathcal B^+ = \{x\in\B  \mid x_i \geq 0, i \in [d]\}$, see \cref{lowerbound:fig:Quarter}. 

Crucially, on $\B^+$ the reward function is almost everywhere zero except inside the exponentially small spherical cap $\mathcal C_\gamma(w) = \{ \| x\|_2 \leq 1 \mid \frac{x^\top w}{\|w\|_2} \geq \gamma \}$ (\cref{lowerbound:fig:Quarter}) for some $w\in\B$. Unless the oracle prescribes an action inside $\mathcal C_\gamma(w)$, \emph{the batch algorithm only observes a zero  reward function} and is unable to distinguish different MDPs using this information. 

\subsection{Setup: State-Action Space and Feature Extractor} 

\paragraph{State space}
The state space $\StateSpace = \{s^\star,\overline s,s^\dagger\}$ consists of a start state $s^\star$, an intermediate state $\overline s$ and a terminal state $s^\dagger$.

\paragraph{Action space} In the starting state $s^\star$ the special action $a^\star$ is available in addition to any action  $a \in \B^+$. In the intermediate state $\overline s$ any action in $ \B^+$ is available but not $a^\star$.
Finally, in the terminal state $s^\dagger$ only $\vec 0 \in \B^+$ is available. Mathematically 
$
\ActionSpace_{s^\star} = \B^+ \cup \{ a^\star \}; \; \ActionSpace_{\overline s} =  \B^+; \; \ActionSpace_{s^\dagger} = \{\vec 0\}.$ 
\paragraph{Feature map}
The feature map only depends on the action:
\begin{align*}
\forall s \in \StateSpace: \;  \phi(s,a) = 
\begin{cases}
[\vec 0,1] & \text{if}\; a = a^\star \quad \text{(only available in $s^\star$).} \\
[a,0] & \text{if}\; a\in \B^+.
\end{cases}
\end{align*}
\subsection{Setup: MDP-specific Rewards and Transitions}
\begin{figure}[b]
\centering  
\hspace{-1.0cm}
\subfigure{\input{fig4.tex}}
\begin{minipage}[b]{0.64\linewidth}
\caption{Action space $\B^+$. On $M_{w,-}$ the reward function is zero for any action $a \in \B^+$; on $M_{w,+}$ it is nonzero only inside $\mathcal C_\gamma(w)$.} 
\label{lowerbound:fig:Quarter}
\end{minipage}
\end{figure}
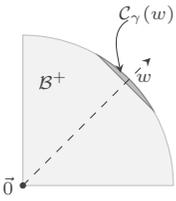
Every MDP $M \in \M$ is identified by a vector $w$ in the outer portion of the positive orthant $\partial\B^+ = \{x\in \B^+ \mid \|x \|_2=1 \}$  and by a $\pm$ sign, and is denoted with $M_{w,+}$ or $M_{w,-}$. 
\paragraph{Transition function}
The transition function $p_w$ depends on the vector $w$ that identifies each MDP in the class, \emph{but not on the sign $+$ or $-$}. Fix the MDP by fixing $w \in \partial \B^+$ (two MDPs correspond to a given choice of $w$). If the agent plays the special action $a^\star$, which is only available in the starting state $s^\star$, it transitions with probability one to the terminal state $s^\dagger$. If the agent plays $a \neq a^\star$, the transition function only depends on the action $a$ (and not on the current state $s$) and the successor state is $\overline s$  with some probability, and is otherwise the absorbing state $s^\dagger$.  

Mathematically, if $a = a^\star$ then $p_w(s^\dagger \mid (s^\star,a^\star)) = 
1$ and if conversely $a \in \B^+$:
\begin{align*}
\begin{cases}
	p_w(\overline s \mid s,a) & = 
\min\{ (1/\gamma) a^\top w, 1 \}, \\ p_w( s^\dagger \mid s,a) & = 1 - p_w( \overline s \mid s,a).
\end{cases}
\numberthis{\label{main:eqn:p}}
\end{align*} 
The definition implies that the successor state is always either $\overline s$ or the terminal state $s^\dagger$.
\paragraph{Reward function}
The reward function $r_{w,\pm}$ depends on both the vector $w \in \partial\B^+$ and on the sign $+$ or $-$ that identifyies the MDP. It is always $\frac{1}{2}$ if the special action $a^\star$ is taken and otherwise it is everywhere $0$ on $M_{w,-}$ or is positive only in the spherical cap on $M_{w,+}$. Mathematically:
\begin{align*}
\text{on} \; M_{w,+}&: \; r_{w,+}(s,a) \defeq \begin{cases} 
\frac{1}{2}\; \; \; \text{if} \; (s,a) = (s^\star,a^\star)\\
\max\{a^\top w - \gamma, 0 \} \; \text{otherwise,} \;
\end{cases} \\
\text{on} \; M_{w,-}&: \; r_{w,-}(s,a) \defeq \begin{cases} 
\frac{1}{2} & \text{if} \; (s,a) = (s^\star,a^\star)\\
0 & \text{otherwise.} \;
\end{cases}
\numberthis{\label{main:eqn:r}}
\end{align*}

\subsection{Proof Sketch of Theorem \ref{lowerbound:thm:BPI:Policy-Induced} (Batch Lower Bound)}
The steps for the proof are the following: we show that 1) $\Qstar$ is linear 2) policy-induced queries are also policy-free for this problem 3) using less than exponentially many queries ensures that at least one spherical cap $\mathcal C_\gamma(\wtilde)$ is not probed 4) the corresponding MDPs $M_{\wtilde,+}$ and $M_{\wtilde,-}$ look the same outside the spherical cap $\mathcal C_\gamma(\wtilde)$ 5) the agent does not have enough information to distinguish $M_{\wtilde,+}$ from $M_{\wtilde,-}$.
\paragraph{Realizability}
By inspection we can verify realizability.
\begin{lemma}[$Q^\star$ is Realizable; \cref{lowerbound:lem:BPI:Policy:Realizability} in appendix]
\label{main:lem:BPI:Policy:Realizability} For any $w\in\partial\B^+$ let $\Qstar_{w,+}$ and $\Qstar_{w,-}$ be the optimal $\Qstar$ values on $M_{w,+}$ and $M_{w,-}$, respectively. It holds that
\begin{align*}
\forall (s,a), \quad \begin{cases}
\Qstar_{w,+}(s,a) = \phi(s,a)^\top [w,\frac{1}{2}] \quad \text{on} \; M_{w,+},\\
\Qstar_{w,-}(s,a) = \phi(s,a)^\top [\vec{0},\frac{1}{2}] \quad \text{on} \; M_{w,-}.
\end{cases}
\end{align*}
\end{lemma}
\paragraph{Policy-free vs policy-induced queries}
Notice that although the dynamics are different for different $w$'s (that identify the MDP), any set $T$ (\cref{lowerbound:asm:PolicyInduced}) induces the same set $\mu$ of state-actions (possibly with the exception of $(s^\star,a^\star)$ and $ (s^\dagger,\vec 0)$) regardless of the vector $w$ and the sign $\pm$. We can therefore consider the case that the oracle has chosen a policy-free query set $\mu = \{(s_i,a_i)\}_{i=1}^{n} \cup \{(s^\star,a^\star), (s^\dagger,\vec 0)\}$. 
\paragraph{Existence of exponentially many spherical caps}
\begin{lemma}[Follows from \cref{lowerbound:lem:LonelyPlus} in appendix]
Assume that less than $2^{-d}\mathcal N(\gamma,d)$ actions $a_1,\dots,a_n$ on $\mathcal B^+$ are selected. Then there exists a spherical cap $\C_\gamma(\wtilde)$ that no action has probed, i.e., $\exists \wtilde \in \partial\B^+ \text{s.t.} \forall i\in[n], a_i \not \in \C_\gamma(\wtilde)$.
\end{lemma}
In this case, the agent does not have any information originating from inside the dark gray spherical cap in \cref{lowerbound:fig:Quarter}. 

\paragraph{$\boldsymbol {M_{w,+}, M_{w,-}}$ are identical outside of the spherical cap}
If $n < 2^{-d}\mathcal N(\gamma,d)$, consider $\wtilde$ given by the above lemma and the two associated MDPs $M_{\wtilde,+}$ and $M_{\wtilde,-}$. Notice that the transition function (\cref{main:eqn:p}) is by construction identical on $M_{\wtilde,+}$ and $M_{\wtilde,-}$ while the reward function (\cref{main:eqn:r}) is zero at any action $a \not\in \mathcal C_{\gamma}(\wtilde)$ selected by the oracle on both $M_{w,+},M_{w,-}$. Then for any $ (s,a)\in \{(s_i,a_i) \}_{i=1}^n$ that the batch algorithm receives it holds that
\begin{align*}
r_{\wtilde,+}(s,a) = r_{\wtilde,-}(s,a) \quad \text{and} \quad p_{\wtilde,+}(s,a)  = p_{\wtilde,-}(s,a).
\end{align*}
\paragraph{Batch algorithm does not have enough information}
The above equation implies that the transitions and the rewards in the dataset could have originated from either $M_{\wtilde,+}$ or $M_{\wtilde,-}$.
In $s^\star$, the batch algorithm has two choices to determine the policy $\pihatstar_{\mathcal D}$ to return: 
choose action $a^\star$ and get a total return of $\frac{1}{2}$ or choose an action $a \neq a^\star$. The second choice  is $\frac{1}{2}$-suboptimal on $M_{\wtilde,-}$, while the first is at least $\frac{1}{2}$-suboptimal on $M_{\wtilde,+}$. At best, the batch algorithm can randomize between the two, showing the result.
\subsection{Proof Sketch of Theorem \ref{lowerbound:thm:BPI:ExponentialSeparation} (Online Upper Bound)}
It remains to exhibit an online algorithm that can solve every problem instance from this MDP class using $d+1$ queries. The algorithm proceeds as follows: 1) it tries to locate the position of the spherical cap by learning the vector $w$ and 2) it probes the spherical cap to learn the $\pm$ sign of the MDP, precisely identifying the MDP.

\paragraph{Identifying the position $\boldsymbol w$ of the spherical cap}
Consider the following adaptive algorithm that submits policy-induced queries. The algorithm first plays the actions $\gamma e_1,\dots,\gamma e_d$ in $s^\star$ where $e_i$ is the vector of all zeros and $1$ in position $i$ (these are $d$ policies that generate trajectories of length one where $\gamma e_i$ is the only action). 

Upon receiving the transition functions $p_w(\overline s \mid s^\star,\gamma e_i)  = 
\min\{ \frac{1}{\gamma} (\gamma e_i)^\top w, 1 \} =  e_i^\top w $ for all $i \in [d]$ (see \cref{main:eqn:p}), the agent can determine each component of the vector $w$. By construction, $w$ identifies the spherical cap $\mathcal C_\gamma(w)$. 
\paragraph{Identifying the sign $\boldsymbol \pm$ of the MDP}
Next, the algorithm plays the state-action $(s^\star,w)$ to probe the spherical cap and observe the reward ($1-\gamma$ on $M_{w,+}$ and $0$ on $M_{w,-}$) which identifies the sign of the MDP. Since the MDP is now precisely identified, the agent can predict the value of any policy, and in particular, it can return the optimal policy.

\section{Discussion}
This work presents exponential lower bounds for batch RL. In general, models are never correct, observations are noisy, and a batch algorithm needs to return an answer using whatever dataset is available; clearly the lower bounds continue to hold in these more general settings as much as they do when using more general predictors (like neural networks) which contain the linear setting as a special case. As these hard instances can only arise in infinite horizon settings, there is an exponential separation between finite and infinite horizon batch RL. 

The strength of our results arise from the `oracle + batch algorithm' protocol which allows us to derive lower bounds for \emph{every} a priori data distribution; as a special case, we recover the concurrent lower bound of \cite{wang2020statistical} for the infinite horizon setting. We highlight that our lower bounds always imply an infinite sample complexity.

Beyond the exponential lower bounds, an important result  
is that online exploration may be required to achieve polynomial sample efficiency on certain RL problems. This is surprising, because online RL has the additional exploration burden compared to batch RL with a good dataset.


Finally, this work helps formalize some of the dangers of the deadly triad, which has long been known to cause algorithmic instabilities and divergence of dynamic programming algorithms. In a sense, the bootstrapping problem is for infinite horizon what the extrapolation problem is for finite horizon MDPs (and finite-steps algorithms), but unlike extrapolation, it cannot be mitigated by adding more samples.

\section*{Acknowledgment}
The author is grateful to Emma Brunskill, Mykel Kochenderfer and Martin Wainwright for providing useful feedback. The author also thanks the reviewers for their helpful and detailed comments. The work was done while the author was visiting the Simons Institute for the Theory of Computing.

%% file: fig1.tex
\begin{tikzpicture}
\def\r{1.6};
\filldraw[color=gray!50, fill=gray!10](0,0) circle (\r);
\draw[blue!75, ->] (0,0) -- node[very near end, below] {\scriptsize $\phi_1$} ({\r*1/sqrt(2)},{\r*1/sqrt(2)});
\draw[darkgray, dashed, ->]  ({\r*1/sqrt(2)},{\r*1/sqrt(2)}) -- node[very near end, sloped, above] {\scriptsize $\gamma\phi^+_1$} (0,{\r*1/sqrt(2)});
\draw[gray, semithick, densely dotted, ->]  (0,0) --  node[midway, left] {\scriptsize $\phi_1 - \gamma\phi^+_1$}(0,{\r*1/sqrt(2)});
\draw[blue!75, ->] (0,0) -- node[very near end, above] {\scriptsize  $\phi_2$} ({\r*1/sqrt(2)},-{\r*1/sqrt(2)});
\draw[darkgray, dashed, ->]  ({\r*1/sqrt(2)},-{\r*1/sqrt(2)}) --  node[very near end, sloped, below] {\scriptsize  $\gamma\phi^+_2$} (0,{-\r*1/sqrt(2)});
\draw[gray, semithick, densely dotted, ->]  (0,0) --  node[midway, left] {\scriptsize $\phi_2 - \gamma\phi^+_2$} (0,{-\r*1/sqrt(2)});
\draw[draw = none,darkgray, thick ,densely dotted, ->] (0,-{1.2*\r}) -- node[pos = 1.1]{} (0,{1.2*\r});
\end{tikzpicture}

%% file: fig2.tex
\begin{tikzpicture}
\def\r{1.6};
\def\dtheta{0.30}
\def\theta{11}
\def\nmax{6};
\filldraw[color=gray!50, fill=gray!10](0,0) circle (\r);
\foreach \n in {1,...,{\nmax}}{
\draw[blue!75,->] (0,0) -- ({\r*cos((\n+\dtheta)/\nmax*360)},{\r*sin((\n+\dtheta)/\nmax*360)});
\draw[darkgray ,dashed, ->] ({\r*cos((\n+\dtheta)/\nmax*360)},{\r*sin((\n+\dtheta)/\nmax*360)}) -- (0,{\r*sin((\n+\dtheta)/\nmax*360)});
\draw[darkgray, thick ,densely dotted, ->] (0,-{1.2*\r}) -- node[at end, right]
{\scriptsize $\mathcal R$}
(0,{1.2*\r});
\draw[red,densely dotted, thick, ->] (({-1.1*\r},{0}) -- node[near end, below ] {\scriptsize $\mathcal N$}
({1.25*\r},{0});
};
\end{tikzpicture}

%% file: fig3.tex
\begin{tikzpicture}
\def\r{1.6};
\def\theta{16}
\filldraw[color=gray!50, fill=gray!10](0,0) circle (\r);
\draw[semithick, gray, fill=gray!50] ({\r*cos(\theta)},{-sin(\theta)*\r}) arc(-\theta:\theta:{\r}) -- cycle;
\draw[semithick, gray, fill=gray!50] ({\r*cos(\theta+180)},{-sin(\theta+180)*\r}) arc(-\theta+180:\theta+180:{\r}) -- cycle;
\draw[blue!75, ->] (0,0) -- node[near end, below] {\tiny $\phi_i$}({\r*0.98},{\r*0.15});
\draw[darkgray, dashed, ->] (({\r*0.98},{\r*0.15}) -- node[near end, above] {\tiny $\gamma\phi^+_i$} ({\r*0.15},{\r*0.15}) ;
\draw[gray, densely dotted, semithick, densely dotted, ->] (({0},{0}) -- node[near end, left] {\tiny $\phi_i - \gamma\phi^+_i$} ({\r*0.15},{\r*0.15});
\draw[draw = none,darkgray, thick ,densely dotted, ->] (0,-{1.2*\r}) -- node[pos = 1.1]{} (0,{1.2*\r});
\end{tikzpicture}

%% file: fig4.tex
\begin{tikzpicture}
\def\r{2.0};
\filldraw[color=gray!50, fill=gray!10] (0,0) -- (\r,0) arc (0:90:\r) -- (0,0);
\filldraw[darkgray] (0,0) circle (1pt) node[anchor=east] {\scriptsize $\vec 0$};
\draw[gray, fill=gray!50] ({\r*cos(30)},{sin(30)*\r}) arc(30:60:{\r}) -- cycle;
\draw[darkgray, dashed, ->] (0,0) -- node[pos = 0.95, below] {\scriptsize $w$} ({\r*1.2/sqrt(2)},{\r*1.2/sqrt(2)});
\node[darkgray] at ({0.2*\r},{0.7*\r}) {\scriptsize $\mathcal B^+$};
\draw[darkgray, ->] ({\r*0.7},{\r*1.1}) to[out=210,in=100] node [pos = -0.25] {\scriptsize $\mathcal C_{\gamma}(w)$} ({\r*0.65},{\r*0.75});
\end{tikzpicture}

%% file: 3-appendix.tex
\section{Benefits of the Setup}
\label{lowerbound:sec:Benfits}
One of the technical innovation of this work lies in the `oracle + batch algorithm' protocol, which allows us to obtain lower bounds that hold for \emph{every} distribution that the oracle can choose; this automatically yields lower bounds for fixed distributions as a special case. The key idea is that if the batch algorithm cannot return a good answer for \emph{any} distribution chosen by the oracle, it certainly cannot return a good answer for an a priori fixed distribution, otherwise the oracle would have chosen it! To formally highlight the strength of the oracle setup, we derive \cite{wang2020statistical}'s lower bound (theorem 4.1 in their paper) for infinite horizon RL as a special case and with an additional remark about infinite data. 

\begin{corollary}[Well Conditioned Covariance Matrix is Insufficient] There exist an MDP class $\M$ and a feature extractor $\phi$ that satisfy \fullref{lowerbound:asm:ALL:MDP}  together with a target policy $\pi$ and a distribution $\mu = \{(s_i,a_i) \}_{i=1}^n$ that induces a covariance matrix 
 $$\frac{1}{d}I = \frac{1}{|\mu|}\sum_{(s,a) \in \mu}\phi(s,a)\phi(s,a)^\top$$ such that no algorithm can predict the value of the target policy $\pi$ with probability $> \frac{1}{2}$ and additive error $ < 1$ (or return a policy with suboptimality $< 1$) even in the limit of infinite data (i.e., sampled rewards and transitions) generated from $\mu$.
\end{corollary}
\begin{proof}
Consider the MDP class $\M$ described in the proof of theorem 4. 
Since any feature vector in the unit Euclidean ball is available, simply choose a distribution $\mu$ that samples the  feature vectors $e_1,\dots,e_d$ (concretely, we can choose $\mu = \{(s^\star,e_1),\dots,(s^\star,e_d) \}$). 

Since $\mu$ is a distribution that the oracle could have chosen and consists of just $|\mu| = d \leq \mathcal N(\gamma,d)$  queries, apply  \fullref{lowerbound:thm:ALL:Policy-Free} 
to deduce that even if infinite data is generated from $\mu$, no batch algorithm can return the value of $Q^\pi(s^\star,\cdot)$ for some action with additive error $< 1$ and with probability $> \frac{1}{2}$;  likewise it cannot return a policy with suboptimality error $<1$	 from $s^\star$ with probability $> \frac{1}{2}$.
\end{proof}

In summary, there exist problems where the batch algorithm cannot return a reasonable answer even with infinite data, the best conditioned covariance matrix and a strong representation for the action value function. We highlight that the exponential lower bound in \cite{wang2020statistical}'s construction  can potentially be avoided by choosing a better batch distribution  since their MDPs are tabular with small state-action spaces; our MDPs are instead much harder and cannot be solved even if one chooses --- through an oracle --- the best distribution for such problem class. Similar results can be derived for any assumption on the covariance matrix. 

Finally, we highlight that all theorems that we present in this work are expressed as a function of the number of state-action queries, i.e., the size of the support of $\mu$, instead of the number of sampled rewards and transitions. This way, if the support of the batch distribution is small $|\mu| \leq 2^{-d}\mathcal N(\gamma,d)$ then our theorems ensure that even infinite samples from the state-actions where $\mu$ is supported do not suffice to make accurate predictions, because the batch algorithm already receives the full reward and transition functions in our interaction protocol. As a clarifying example, consider choosing a sampling distribution $\mu$ according to globally-optimal design of experiments \citep{pukelsheim2006optimal}, which is an optimal procedure in linear bandits to recover the unknown parameter \citep{lattimore2020bandit}. Design of experiments fails to be effective in our hard instances because it prescribes only $O(d^2)$ distinct state-actions, i.e., queries, where to (repeatedly) acquire samples, but to escape the lower bound one does not need more samples at the same state-actions; instead \emph{the support of the sampling distribution needs to grow}.

\section{Related Literature} 
\label{app:Literature}

For tabular MDPs one query in each state and action pair with exact feedback is sufficient to identify the MDP. The available results for learning in tabular MDPs consider a similar setting, where only noisy observations are available \citep{Azar12,li2020breaking,agarwal2019model}. The algorithms examined by these authors are batch algorithms which specify a uniform distribution over all the state-actions and are minimax optimal. Adaptive algorithms with a generative model \citep{zanette2019b,marjani2020best} exist but offer  little minimax advantage.\footnote{They can however remove the explicit dependence on the action space from the main rate even in minimax problems, see \cite{zanette2019b}, section $6$.} In the following discussion we primarily consider the function approximation setting with linear value functions.

\paragraph{Negative results: lower bounds}
Polynomial lower bounds are routinely derived to certify that a certain algorithm is sample efficient. For example, \cite{jiang2016doubly} present lower bounds for tabular RL and \cite{duan2020minimax,hao2020sparse} for the linear setting under additional closure assumptions on the Bellman operator. Instead, in this work we are interested in exponential lower bounds.

Dynamic programming algorithms like least-square value and policy iteration \cite{bradtke1996linear,lagoudakis2003least} are widely adopted but they are prone to divergence \cite{baird1995residual,tsitsiklis1996feature}. The hardness of obtaining statistically efficient algorithms has motivated researchers to look for information-theoretic lower bounds. In particular, \cite{chen2019information} provide a generic lower bound in absence of concentrability; such lower bound does not apply to our setting as we further assume a linear structure. \citet{du2019good} show that if a linear predictor is highly misspecified then an exponential number of queries is needed, and \citet{lattimore2020learning} generalize this result through a corollary to show that using $\epsilon$-accurate linear predictors in dimension $d$ can only give $\sqrt{d}\epsilon$-optimal policies for linear bandits and RL. While these results are relevant to our setting, they becomes vacuous in absence of misspecification, i.e., when realizability holds. We examine the relation with concurrent work \cite{weisz2020exponential,wang2020statistical,amortila2020variant} in  \cref{sec:MainResult}; however, here we mention that their constructions are easily solvable under our assumptions.

\paragraph{Positive results: upper bounds}

When the action value function has a linear parameterization and the transitions are low rank, \citet{yang2020reinforcement,jin2020provably,zanette2020frequentist} propose regret-minimizing online algorithms; these works implicitly imply that both the BPI and OPE batch problems are easily solvable under low-rank dynamics. For the more general setting that the value function is closed under the Bellman operator, batch upper bounds for the BPI problem  exist \cite{munos2005error,munos2008finite} and these have been recently generalized to the online setting  \cite{zanette2020learning} with a computationally tractable algorithm \cite{zanette2020provably}; collectively, these works show that online RL can operate under batch assumptions, and our work shows that online RL is in fact exponentially easier than batch RL on certain problems.  Minimax results for batch OPE problems are also available \cite{duan2020minimax} for a  setting essentially equivalent to low inherent Bellman error; for the BPI problem see instead \cite{xie2020Q}. Batch methods based on minimizing the Bellman residual also exist \cite{antos2008learning}. When the action value function of all policies can be linearly represented \citet{munos2003error,lazaric2012finite,lattimore2020learning,agarwal2020optimality} provide algorithms to learn a good policy using a simulator, and online using stronger conditions \cite{agarwal2020pc}; all these methods are non-batch, and our \cref{lowerbound:thm:ALL:Policy-Free} rules out the existence of batch algorithms operating under the same assumptions. Other linear models recently considered include \cite{ayoub2020model,zhou2020provably} and when the Bellman equations are linearly representable up to a low-rank error \cite{jiang17contextual}; although the focus of these works is online exploration, learning should be possible in the batch setting as well.  

The optimal policy (as opposed to a near optimal one) can be identified if there exists some separation between the best and the second best policy \cite{du2020agnostic}, although a sample complexity proportional to the inverse gap (which can be exponentially small) must be suffered; learning in such setting is also possible with an entirely batch algorithm. Near-convexity also ensures that both the BPI and OPE learning problems are solvable in the batch setting \cite{zanette19limiting}; under exact convexity  \cite{cui2020plug} give an optimal sample complexity for the BPI problem. Deterministic systems with linear value functions are also learnable \cite{WR13} in finite horizon, and it is easy to derive a batch algorithm for both the BPI and OPE problem under this  assumption. \citet{xie2020batch} show that under strong concentrability assumptions the BPI problem  is solvable using only realizability and general (non-linear) function approximators and \citet{liu2020provably} show how to find the best-in-class solution to the BPI problem while avoiding concetrability for general approximators. Finally, importance sampling estimators \cite{precup2000eligibility,li2015toward} make very little assumptions but as such they do not leverage the linear structure of the problem and can exhibit unbounded variance. However, methods to reduce the variance exist \cite{jiang2016doubly,thomas2016data,liu2018breaking,xie2019towards}.

\section{Globally Optimal Experimental Design is Not a Good Sampling Strategy in RL}
\label{app:DoE}

We wonder if the classical globally optimal design of experiment (DoE)  \cite{pukelsheim2006optimal} from statistics produces acceptable  results in light of our lower bounds, for example, whether it yields a matching exponential upper bound. DoE prescribes a set of states and actions (the design) where to acquire several samples; multiple samples are acquired at the selected state-actions, according to the design. The procedure thus prescribes policy-free queries.
In finite horizon DoE is successful: \citet{weisz2020exponential} conclude that DoE with value iteration can identify a near optimal policy in finite horizon given exponentially many samples. 

Dishearteningly, we find that in infinite horizon this is no longer true: in the limit of infinite samples over the feature vectors selected by DoE the learner receives exact feedback, but since DoE always prescribes at most $d(d+1)/2$ distinct vectors where samples are acquired (see \cite{lattimore2020bandit,pukelsheim2006optimal}), the support of the query set  is always at most $d(d+1)/2$ regardless of the number of actual samples along these feature vectors. For large enough $d$ and $\gamma$ close to $1$ we have $d(d+1)/2 \leq I_{1-\gamma^2}^{-1}(\frac{d-1}{2},\frac{1}{2})$ as the left hand side is polynomial and the right hand side is exponential. 
In light of \cref{lowerbound:thm:ALL:Policy-Free}, we conclude that even if infinite data are collected over the feature vectors chosen by DoE and all policies have a linear representation, no batch algorithm can output a good policy or estimate the value of another with good enough probability. This happens because traditional DoE fails to account for the effect of bootstrapping, as we explain in \cref{main:sec:Intuition}. 
Thus, in infinite horizon the challenge is much greater than in finite horizon, because a good policy (or the value of a target policy) cannot eventually be learned by just collecting more samples, instead, \emph{the support of the sampling distribution needs to grow}.

\section{Preliminaries}
\subsection{Additional Notation}
\label{app:additional_notation}
We indicate with $\B$ the Euclidean ball in dimension $d$, i.e., $\B \defeq \{x \in \R^d \mid \|x \|_2 \leq 1 \}$.  
We denote with $\partial \B \defeq \{x \in \R^d \mid \|x \|_2 = 1 \}$ the surface of the unit ball. In addition, let $e_i = [0,0,\dots,0,1,0,\dots,0]$ be the $i$-th canonical vector in $\R^d$, i.e., the vector of all zeros with a $1$ in the $i$-th position. We define the unit Euclidean ball restricted to the positive orthant: $\B^+ \defeq \{x\in \B \mid x^\top e_i\geq 0, \forall i \in [d] \} $ and its outer surface $\partial \B^+ \defeq \{x\in \partial\B \mid x^\top e_i\geq 0, \forall i \in [d] \}$. Sometimes we indicate with $\vec 0$ the zero vector in $\R^d$ to distinguish it from the $0\in\R$. We occasionally write $r_M,p_M,V_M,Q_M,\T_M,\pi_M$ to highlight the dependence on a certain MDP $M$.

In the proofs for the lower bounds we consider a class $\M$ of MDPs containing MDP instances $M$ with certain properties. Every MDP in the class $\M$ can be identified by a vector $w$ and a sign $+$ or $-$, and is denoted by $M_{w,+}$ or $M_{w,-}$, respectively. We overload the notation slightly and write $M_{w}$ in a statement to indicate the statement holds for both $M_{w,+}$ and $M_{w,-}$. This allows us to write, for example, $\M = \{M_{w,+} \mid w \in \partial \B \} \cup \{M_{w,-} \mid w \in \partial \B \} = \{M_{w} \mid w \in \partial \B \}$.

The volume of a (measurable) set $\Omega \in \R^d$ is defined as $\Vol(\Omega) = \int_{x \in \Omega}dx$.

The \emph{gamma function}  for a non-negative $z\in\R$ is defined as 
$$\Gamma(z) \defeq \int_0^\infty x^{z-1}e^{-x}dx$$ 
and the \emph{beta function} for non-negative $a,b \in \R$ is defined as 
$$B(a,b) \defeq \int_{0}^{1}t^{a-1}(1-t)^{b-1}dt.$$
If $x \in [0,1]$ we can define the \emph{incomplete beta function} 
$$B(x,a,b) = \int_{0}^{x}t^{a-1}(1-t)^{b-1}dt$$ which allows us to define  the regularized incomplete beta function $$I_{x}(a,b) \defeq \frac{B(x,a,b)}{B(a,b)}.$$
Useful bounds are provided in \fullref{lowerbound:lem:UpperBound}.

\subsection{Learning Finite Horizon MDPs with Policy-Free Queries}
In this section we describe the process to learn a finite horizon MDPs with a linear representation $Q_t^\pi(s,a) = \phi_t(s,a)^\top\theta^\pi_t$ for the target policy $\pi$ and $Q_t^\star(s,a) = \phi_t(s,a)^\top\theta^\star_t$ for the optimal action-value function, $\forall t \in [H]$. We make the same assumptions as in the main text.

Consider the following learner that submits policy-free queries. First, in every timestep $t\in[H]$ the oracle selects  appropriate state-actions $(s_1,a_1),\dots,(s_d,a_d)$ such that the corresponding feature vectors $\phi_t(s_1,a_1),\dots,\phi_t(s_d,a_d)$ form a basis for $\R^d$  (if a basis does not exist it is possible to reduce the dimensionality of the problem). Then for the BPI problem, the batch algorithm runs least-square value iteration backward. By induction, the algorithm identifies $\theta^\star_t$ for all $t$. In particular, assume $\theta^\star_{t+1}$ is known, which is certainly true for $t+1 = H+1$ (the base case). Then the batch algorithm solves in every timestep $t$ the following linear system of equations ($\forall i\in[d]$):
\begin{align}
	\phi_t(s_i,a_i)\theta_t & = r(s_i,a_i) + \E_{s'\sim p(s_i,a_i)} \max_{a'\in \ActionSpace_{s'}} \phi(s',a')^\top \theta^\star_{t+1} \\
	& = r(s_i,a_i) + \E_{s'\sim p(s_i,a_i)} \Vstar_{t+1}(s') \\
	& = \Qstar_{t}(s_i,a_i). 
\end{align}
Since $\phi_t(s_i,a_i), i = 1,\dots,d$ form a basis for $\R^d$, $\theta_t = \theta_t^\star$ is the unique solution. By induction, we conclude that the BPI problem can be solved in finite horizon MDPs under our assumptions.

The OPE problem is analogous. The batch algorithm solves the following linear system of equations for $i \in [d]$ at timestep $t$.
\begin{align}
	\phi_t(s_i,a_i)\theta_t & = r(s_i,a_i) + \E_{s'\sim p(s_i,a_i)} \phi(s',\pi(s'))^\top \theta^\pi_{t+1} \\
	& = r(s_i,a_i) + \E_{s'\sim p(s_i,a_i)} \Vpi_{t+1}(s') \\
	& = \Qpi_{t}(s_i,a_i).
\end{align}
Since $\phi_t(s_i,a_i), i = 1,\dots,d$ form a basis for $\R^d$, $\theta_t = \theta_t^\pi$ is the unique solution. By induction, we conclude that the OPE problem can be solved in finite horizon MDPs under our assumptions.

\subsection{Hypersphere and Hyperspherical Sectors}
\label{sec:hyper}
\begin{definition}[Hypersphere, hyperspherical cap and hyperspherical sector]
\label{lowerbound:def:hyper}
Fix $b \geq 0$ and define the $b$-hyperspherical cap in direction $w\in \B$ as
\begin{align}
\C_b(w) \defeq \{ x \in \B \mid \frac{x^\top w}{\|w\|_2} \geq b \}
\end{align}
and the $b$-hyperspherical sector in direction $w$ as
\begin{align}
\Ct_b(w) \defeq \Big\{ x \in \B \mid  \frac{x^\top w}{\| x \|_2\|w\|_2}  \geq b \Big\}.
\end{align}
\end{definition}

The following formulas are useful to compute the volume of the hypersphere $\B$ and of a spherical sector $\Ct_b(w)$.

\begin{fact}[Volume of a Hyperspherical Sector]
\label{fact:VolumeSector}
The volume of an spherical sector is given by the formula
\begin{align}
\Vol(\Ct_b(w)) = \frac{\Vol(\B)}{2}I_{1-b^2}\(\frac{d-1}{2},\frac{1}{2}\).
\end{align}
where the volume of the ball $\B$ is
\begin{align}
\Vol(\B) = \frac{\pi^{d/2}}{\Gamma(\frac{d}{2}+1)}.
\end{align} 
\end{fact}
\begin{proof}
For the proof, see \cite{li2011concise} where $\phi$ in their notation is the half-angle of the hypersector and $\sin^2 \phi = 1 - \cos^2 \phi = 1- b^2$.
\end{proof}

\subsection{Bounds on the Regularized Incomplete Beta Function}

\begin{lemma}[Upper Bound on $I$]
\label{lowerbound:lem:UpperBound}
The following upper bound holds true for $\gamma \in (0,1)$ and $d \geq 3$:
\begin{align}
I_{1-\gamma^2}(\frac{d-1}{2},\frac{1}{2})\leq \Idef.
\end{align}
\end{lemma}
\begin{proof}
We need to compute an upper bound on
\begin{align}
I_{1-\gamma^2}\(\frac{d-1}{2},\frac{1}{2}\) \defeq \frac{B(1-\gamma^2,\frac{d-1}{2},\frac{1}{2})}{B(\frac{d-1}{2},\frac{1}{2})}.
\end{align}
We first compute an upper bound on the incomplete Beta function
\begin{align}
B(1-\gamma^2,\frac{d-1}{2},\frac{1}{2}) & \defeq \int_{0}^{1-\gamma^2} t^{\frac{d-1}{2}-1}(1-t)^{\frac{1}{2}-1} dt \\
& = \int_{0}^{1-\gamma^2} \sqrt{\frac{t^{d-3}}{(1-t)} }dt.
\end{align}
Notice that $t \leq 1-\gamma^2$ and so $1-t \geq \gamma^2$ and finally $\frac{1}{1-t} \leq \frac{1}{\gamma^2}$. Therefore the following upper bound follows:
\begin{align}
B(1-\gamma^2,\frac{d-1}{2},\frac{1}{2}) & \leq \frac{1}{\gamma} \int_{0}^{1-\gamma^2} t^{\frac{d-1}{2}-1}dt \\
& = \frac{1}{\gamma} \times \frac{2}{d-1} \times  \evalat{t^{\frac{d-1}{2}}}{t = 1-\gamma^2} \\
& = \frac{1}{\gamma} \times \frac{2}{d-1} \times  (1-\gamma^2)^{\frac{d-1}{2}}.\end{align}
It remains to compute a lower bound on the Beta function:
\begin{align}
B(\frac{d-1}{2},\frac{1}{2}) & = \frac{\Gamma(\frac{d-1}{2})\Gamma(\frac{1}{2})}{\Gamma(\frac{d-1}{2} + \frac{1}{2})}.
\end{align}
We have $\Gamma(\frac{1}{2}) = \sqrt{\pi}$; we thus focus on the ratio
\begin{align}
\Bigg[ \frac{\Gamma(\frac{d-1}{2})}{\Gamma(\frac{d-1}{2} + \frac{1}{2})} \Bigg]^{-1} & = \frac{\Gamma(\frac{d-1}{2} + \frac{1}{2})}{\Gamma(\frac{d-1}{2})} \\
&  = \frac{\Gamma(\frac{d-2}{2} + 1)}{\Gamma(\frac{d-2}{2} + \frac{1}{2})} \\
& \leq \( \frac{d-2}{2} + 1 \)^{1 - \frac{1}{2}} \\
& = \sqrt{\frac{d}{2}}.
\end{align}
The inequality is Gautschi's inequality for the gamma function.
Therefore, we conclude
\begin{align}
B(\frac{d-1}{2},\frac{1}{2}) \geq \sqrt{\frac{2\pi}{d}}.
\end{align}
The bounds just derived for the beta functions together yield an upper bound on the regularized incomplete beta function:
\begin{align}
I_{1-\gamma^2}\(\frac{d-1}{2},\frac{1}{2}\) \leq \Idef.
\end{align}
\end{proof}

Using the above result we can derive the following easily interpretable lower bounds.

\begin{corollary}[Lower Bounds on $I^{-1}$]
\label{cor:LowerBound}
If $d \geq 5$ then the following lower bounds holds true for $\gamma \in (0,1)$:
\begin{align}
\lb
\end{align}
\end{corollary}
\begin{proof}
Using \cref{lowerbound:lem:UpperBound} we can derive the following crude lower bound for $d \geq 5$
\begin{align}
I^{-1}_{1-\gamma^2}\(\frac{d-1}{2},\frac{1}{2}\)
& \geq \gamma \sqrt{\frac{\pi}{2}} \times \frac{d-1}{\sqrt{d}} \times \(\frac{1}{(1+\gamma)(1-\gamma)}\)^{\frac{d-1}{2}} \\
& \geq \gamma \sqrt{d} \(\(\frac{1}{2}\)\cdot\frac{1}{(1-\gamma)}\)^{\frac{d-1}{2}}.
\end{align}
Under the same conditions we have the following lower bound 
\begin{align}
2^{-d}I^{-1}_{1-\gamma^2}\(\frac{d-1}{2},\frac{1}{2}\) & \geq \gamma \sqrt{d}\(\(\frac{1}{2}\)^{\frac{2d}{d-1}} \cdot \( \frac{1}{2}\) \frac{1}{(1-\gamma)}\)^{\frac{d-1}{2}} \\
& \geq \gamma \sqrt{d} \(\(\frac{1}{2}\)^{3.5} \cdot \frac{1}{(1-\gamma)}\)^{\frac{d-1}{2}}.
\end{align}
\end{proof}

\newpage
\section{Existence of a Lonely Hyperpherical Cap}
Consider a set of $n$ points in the unit ball $\{y_1,\dots,y_n \} \subseteq \B$. In this section we show that if $n$ is less than $N_\gamma \approx \(\frac{1}{1-\gamma}\)^d$ then there exists an hyperspherical cap $\C_\gamma(w)$ (identified by its direction $w$) such that none of the $y_i$'s is in $\C_\gamma(w)$ or its symmetric counterpart $\C_\gamma(-w)$. For short, define
\begin{align}
N_\gamma \defeq \frac{2}{I_{1-\gamma^2}\(\frac{d-1}{2},\frac{1}{2}\)} = 2\mathcal N(\gamma,d).
\end{align}

\begin{lemma}[Existence of a Lonely Hyperspherical Cone]
\label{lowerbound:lem:Lonely}
Let $\mu = \{y_1,\dots,y_n\} \subseteq \B$ be a collection of $n$ points. If $n < N_\gamma$ then there exists a point $\wtilde \in \partial\B$ such that its $\gamma$-spherical cone does not contain any of the $y_i$'s, i.e.,
\begin{align}
\forall \mu =\{y_1,\dots,y_n\} \subseteq \B \quad \text{s.t.} \quad n < N_\gamma, \; \exists \wtilde \in \partial \B \quad \text{such that}\quad \forall y \in \{y_1,\dots,y_n\}, \; y \not \in \Ct_\gamma\( \wtilde\).
\end{align}
\end{lemma}
\begin{proof}
This follows from a geometrical argument; the idea is that the hyperspherical cones around $y_1,\dots,y_n$ with parameter $\gamma$ are not sufficient to cover the whole hypersphere, leaving a ``gap''. A point in that gap, which we denote\footnote{We save the notation $\wtilde$ for its normalization, i.e., $\wtilde = \frac{w}{\| w\|_2}$.} with $w \neq 0$  is not covered by any of the hyperspherical cones 
$\{\Ct_\gamma\(y_1\),\dots, \Ct_\gamma\( y_n\) \} \subseteq \B$, which in turn means that the hyperspherical cone $\Ct_\gamma(w)$ around such point $w$ cannot cover any of the $y$'s.

More formally, consider the hyperspherical cones $\Ct_\gamma\(y_1\),\dots, \Ct_\gamma\( y_n\)$. The volume of their union is at most
\begin{align}
\text{Vol}\(\Ct_\gamma\(y_1\) \cup \dots \cup \Ct_\gamma\( y_n\)  \) \leq \sum_{i=1}^n \text{Vol}\( \Ct_\gamma(y_i)\) = n \text{Vol}\(\Ct_\gamma\(y_1\)\).
\end{align}

We have
\begin{align}
 \frac{\Vol(\B)}{\text{Vol}\(\Ct_\gamma\(y_1\)\)} =  \frac{2}{I_{1-\gamma^2}\(\frac{d-1}{2},\frac{1}{2}\)} = N_\gamma
\end{align}
where the first equality follows from \fullref{fact:VolumeSector}.
It is easily seen that if $n < N_\gamma$ then there must exist a point $w \neq 0, w \in\B$ (in fact, a whole subset of $\B$ of non-zero volume) not covered by any spherical sector around the $y$'s, i.e.,
\begin{align}
\exists w \in \B \quad  \text{such that} \quad w \not \in \bigcup_{i=1}^n \Ct_\gamma(y_i).
\end{align}

Now denote with $\wtilde \defeq \frac{w}{\| w \|_2}$ its normalization; it follows by definition that $\wtilde \not \in \bigcup_{i=1}^n \Ct_\gamma(y_i)$ and $\wtilde \in \partial B$.  Consider the spherical cone around $\wtilde$, i.e., consider $\Ct_\gamma(\wtilde)$. By symmetry we have that none of the $y_i$'s can be in $\Ct_\gamma(\wtilde)$. This is because $\wtilde \not \in \Ct_\gamma(y_i)$ means
\begin{align}
\frac{y_i^\top\wtilde}{\| y_i \|_2\| \wtilde \|_2}  < \gamma 
\end{align} 
which is equivalent to saying $y_i \not \in \Ct_\gamma(\wtilde)$, and this can be repeated for every $i \in [n]$. 
\end{proof}

We also need the following closely related result in the positive orthant.
\begin{lemma}[Existence of a Lonely Spherical Cone in the Positive Orthant]
\label{lowerbound:lem:LonelyPlus}
Let $\mu = \{y_1,\dots,y_n\} \subseteq \B^+$ be a collection of $n$ points. If $n < 2^{-d}N_\gamma$ then there exists a point $\wtilde \in \partial\B^+$ that also satisfies $e^\top_i \wtilde > 0, \; \forall i \in[d]$ such that its $\gamma$-spherical cone does not contain any of the $y_i$'s, i.e.,
\begin{align}
\forall \mu =\{y_1,\dots,y_n\} \subseteq \B^+, \; n < 2^{-d}N_\gamma, \; \exists \wtilde \in \partial\B^+, \; e^\top_i \wtilde > 0 \; \forall i \in[d], \\
\quad \text{such that}\quad \forall y \in \{y_1,\dots,y_n\}, \; y \not \in \C_\gamma\( \wtilde\).
\end{align}
\end{lemma}
\begin{proof}
The proof for the positive orthant (i.e., restricted to $\B^+$) is nearly identical to \fullref{lowerbound:lem:Lonely}.
Consider the hyperspherical sectors intersected with $\B^+$, i.e., $\Ct_\gamma\(y_1\) \cap \B^+,\dots, \Ct_\gamma\( y_n\)\cap \B^+$. The volume of their union is at most
\begin{align}
\text{Vol}\((\Ct_\gamma\(y_1\) \cap \B^+) \cup \dots \cup (\Ct_\gamma\( y_n\)\cap \B^+)  \) & \leq \sum_{i=1}^n \text{Vol}\( \Ct_\gamma(y_i) \cap \B^+\) \\
& \leq \sum_{i=1}^n \text{Vol}\( \Ct_\gamma(y_i) \) \\
& = n \text{Vol}\(\Ct_\gamma\(y_1\)\).
\end{align}

From geometry we know that the volume of the hypersphere in its positive orthant is $2^{-d}$ times the volume of the hypersphere: 
\begin{align}
\frac{\Vol(\B^+)}{\text{Vol}\(\Ct_\gamma\(y_1\)\)} = \frac{2^{-d}\Vol(\B)}{\text{Vol}\(\Ct_\gamma\(y_1\)\)} =  2^{-d}\frac{2}{I_{1-\gamma^2}\(\frac{d-1}{2},\frac{1}{2}\)} = 2^{-d}\cdot N_\gamma.
\end{align}
where the second equality follows from \fullref{fact:VolumeSector}.
It is easily seen that if $n < 2^{-d}N_\gamma$ then there must exist a whole subset of $\B^+$ of non-zero volume  not covered by any hyperspherical cone around the $y$'s, which allows us to claim
\begin{align}
\exists w \in \B^+, \; w^\top e_i > 0, \forall i\in[d],  \quad  \text{such that} \quad w \not \in \bigcup_{i=1}^n \Ct_\gamma(y_i).
\end{align}

Now denote with $\wtilde \defeq \frac{w}{\| w \|_2}$ its normalization; it follows by definition that $\wtilde \not \in \bigcup_{i=1}^n \Ct_\gamma(y_i)$ and $\wtilde \in \partial \B^+, \; \wtilde^\top e_i > 0, \forall i\in[d]$.  Consider the hyperspherical cone around $\wtilde$, i.e., consider $\Ct_\gamma(\wtilde)$. By symmetry we have that none of the $y_i$'s can be in $\Ct_\gamma(\wtilde)$. This is because $\wtilde \not \in \Ct_\gamma(y_i)$ means
\begin{align}
\frac{y_i}{\| y_i \|_2}^\top \frac{\wtilde}{\|  \wtilde \|_2} < \gamma 
\end{align} 
which is equivalent to saying $y_i \not \in \Ct_\gamma(\wtilde)$, and this can be repeated for every $i \in [n]$. Since $\C_\gamma(\wtilde) \subset  \Ct_\gamma(\wtilde)$, the thesis follows.
\end{proof}
Building on \fullref{lowerbound:lem:Lonely}, we can actually show there exists \emph{two symmetric} hyperspherical sectors that do not contain any of $y$'s. 

\begin{lemma}[Existence of two Symmetric Hyperspherical Cones]
\label{lowerbound:lem:Existence}
Let $\{y_1,\dots,y_n\} \subseteq \B$ be a collection of $n$ points. If $n < N_\gamma/2$ then there exists a point $\wtilde \in \partial\B$ such that the $\gamma$-hyperspherical cones around $\wtilde$ and $-\wtilde$ do not contain any of the $y's$, i.e., 
\begin{align}
\exists \wtilde \in \partial\B \quad \text{such that}\quad y_i \not \in \Ct_\gamma\( \wtilde\) \cup \Ct_\gamma\( -\wtilde\), \quad \quad \forall  i \in [n],
\end{align}
so in particular, $y_i \not \in \C_\gamma\( \wtilde\) \cup \C_\gamma\( -\wtilde\), \forall i \in [n]$.
\end{lemma}
\begin{proof}
Consider the augmented set $\{y_1,\dots,y_n,-y_1,\dots,-y_n\}$. Then \fullref{lowerbound:lem:Lonely} applied to this set ensures that if $2n < N_\gamma$ then there exists a point $\wtilde \in \B$ with unit norm $\| \wtilde \|_2 = 1$ such that 
\begin{align}
y \not \in \Ct_\gamma\( \wtilde\), \quad \quad \forall y \in \{y_1,\dots,y_n, -y_1,\dots,-y_n\}.
\end{align}
This means
\begin{align}
\frac{y_i^\top \wtilde}{\| y_i \|_2 } & < \gamma \rightarrow y_i^\top \wtilde <  \| y_i \|_2\gamma  < \gamma\\
\frac{-y_i^\top \wtilde}{\| y_i \|_2 } & < \gamma \rightarrow -y_i^\top \wtilde < \| y_i \|_2 \gamma < \gamma.
\end{align} 
and so in particular, $y_i \not \in \C_\gamma\( \wtilde\) \cup \C_\gamma\( -\wtilde\), \forall i \in [n]$.
\end{proof}

\newpage
\section{Proof of Theorem \ref{lowerbound:thm:OPE:Policy-Induced}}

\subsection{MDP Class}
\label{sec:OPE:Policy:class}

We consider a class $\M$ of MDPs sharing the same state-space $\mathcal S$, action space $\ActionSpace$, discount factor $\gamma$, and transition function $p$. The MDPs differ only in the reward function $r$ and the prescribed target policy. 
We first describe the state and action space, the discount factor and the transition probabilities, which are fixed across all MDPs in the class $\M$.
\paragraph{State space} 
Each state $s$ in the state space $\StateSpace$ can be identified by a point in the Euclidean ball, i.e., we write $\StateSpace = \B$. The starting state is the origin $s^\star = \vec 0\in \B$.

\paragraph{Action space}
In each state $s \in \StateSpace$ the action set coincides with the unit ball, i.e., $\forall s\in \StateSpace, \ActionSpace_s = \B$.  

\paragraph{Discount factor}
The discount factor $\gamma$ is in the interval $(0,1)$.

\paragraph{Feature map}
The feature extractor returns the point in the Euclidean ball corresponding to the action chosen in the selected state. Mathematically  $$\forall (s,a),\;
\phi(s,a)  = a.$$
Since the available actions are always a subset of $\B$, it is easy to see that the image of the feature map is the set $\B$ and $\|\phi(\cdot,\cdot)\|_2 \leq 1$ holds.

\paragraph{Transition Function}
The transition function is deterministic and is only a function of the action chosen in the current state; it is thus convenient to denote with $s^+(a)$ the unique successor state reached upon taking action $a$ (in any state $s$). The successor state is equivalent to the action taken, i.e., 
\begin{align}
\label{eqn:TransitionFunctionMultiOffPolicy}
\successor(a) = a.
\end{align} 
Since $a \in \B$, we have that $\successor(a) \in \B$, which is a valid state, and so the above display identifies a valid transition function.

\subsection{Instance of the Class}
\label{sec:OPE:Policy:instance}
Every MDP in the class $\M$ can be identified by a vector $w \in \partial \B$ and a sign $+$ or $-$, and is denoted by $M_{w,+}$ or $M_{w,-}$, respectively. 
We now describe the MDP-specific reward function and target policy.

\paragraph{Target policy}

The target policy depends on the vector $w\in \partial\B$ that identifies the MDP $M_{w}$, but not on the sign $+$ or $-$ that distinguishes $M_{w,+}$ from $M_{w,-}$ (therefore, knowledge of the target policy does not reveal the exact MDP in the class). In addition, the target policy is deterministic.
Let $\C_\gamma(w)$ be a $\gamma$-hyperspherical cap in direction $w$ (see \fullref{lowerbound:def:hyper}). Then the target policy is defined as  
\begin{align}
\label{eqn:TargetPolicy}
\pi_w(s) = 
\begin{cases}
\frac{1}{\gamma} (s^\top w) w, \quad \quad \quad & \text{if} \; s\not \in \C_\gamma(w)\cup \C_\gamma(-w), \\
s \quad \quad \quad & \text{otherwise}.
\end{cases}
\end{align} 

Since $s \in \B$ and $w \in \partial\B$, the linear algebra operations in the above definition are all well defined.
In addition, $\text{if} \; s\not \in \C_\gamma(w)\cup \C_\gamma(-w)$ we have $\| \frac{1}{\gamma} (s^\top w) w \|_2 = \frac{|s^\top w|}{\gamma}  \|  w \|_2 \leq \frac{\gamma}{\gamma}\| w \|_2  = 1$ using the definition of hyperspherical cap in \fullref{lowerbound:def:hyper}, so $\frac{1}{\gamma} (s^\top w) w \in \B$ as well. This means the action chosen by the target policy $\pi_w(s)$ lives in $\B$ and it is thus a valid action, so the definition of the target policy in the above display is well posed. 

\paragraph{Reward function}
The reward function depends on both the vector $w \in \partial\B$ and on the sign $+$ or $-$ identifying the MDP and it is defined as follows:
\rdef

Notice that this is a valid definition as the linear algebra operations are well defined. 

\begin{remark}[Opposite MDPs]
The prescribed target policy is the same on $M_{w,+}$ and $M_{w,-}$. In particular, the two MDPs have identical transition functions but opposite reward functions. In addition, notice that $M_{w,+} = M_{-w,-}$ and $M_{w,-} = M_{-w,+}$.
\end{remark}
\begin{remark}[Planning to Explore]
\label{rem:Planning}
Since the dynamics are the same for all MDPs $\in \M$, and the oracle has access to each one of them, it must know the  dynamics, i.e., the oracle knows how to navigate the environment.
\end{remark}

\subsection{Realizability}
We compute the action value function for each MDP (and its associated policy) in the class, showing realizability. 
\begin{lemma}[$Q^\pi$ Realizability]
\label{lowerbound:lem:OPE:Policy:Realizability}
For any $w \in \partial\B$ let $Q_{w,+}$ and $Q_{w,-}$ be the action-value functions of $\pi_w$ (the target policy) on $M_{w,+}$ and $M_{w,-}$, respectively. Then it holds that
\begin{align}
\label{eqn:OPE:Policy:Solution}
\forall (s,a), \quad \begin{cases}
Q_{w,+}(s,a) = \phi(s,a)^\top (+w) \quad \text{on} \; M_{w,+},\\
Q_{w,-}(s,a) = \phi(s,a)^\top (-w) \quad \text{on} \; M_{w,-}.
\end{cases}
\end{align}
\end{lemma}
\begin{proof}
Consider $M_{w,+}$. At all $(s,a)$, $Q_{w,+}$ must satisfy the Bellman evaluation equations for $\pi_w$. In particular, consider applying $\T^{\pi_w}_{w,+}$ to $Q_{w,+}$
\begin{align*}
(\T^{\pi_w}_{w,+}Q_{w,+})(s,a) & = r_{w,+}(s,a) + \gamma Q_{w,+}(\successor_{w}(a),\pi_w(\successor_{w}(a))) \\
& = r_{w,+}(s,a) + \gamma Q_{w,+}(a,\pi_w(a))
\numberthis{\label{eqn:OPE:Policy:Bellman}}
\end{align*}
since $\successor_{w}(a) = a$ is the only possible successor state.  We now evaluate the RHS. Two cases are possible: if $a \not \in \C_\gamma(w)\cup \C_\gamma(-w)$ then the reward function is zero and the rhs of the above equation reads
\begin{align}
\gamma Q_{w,+}(a,\pi_w(a)) & = \gamma \phi(a,\pi_w(a))^\top (+w) \\
& = \gamma \( \frac{1}{\gamma} (a^\top w) w \)^\top (+w) \\
& = a^\top w \\
& = \phi(s,a)^\top w \\
& = Q_{w,+}(s,a). 
\end{align}
If conversely $a \in \C_\gamma(w)\cup \C_\gamma(-w)$ then the rhs reads
\begin{align}
(1-\gamma)a^\top w + \gamma Q_{w,+}(a,\pi_w(a)) & = (1-\gamma)a^\top w  + \gamma \phi(a,\pi_w(a))^\top (+w) \\
& = (1-\gamma)a^\top w + \gamma a^\top w  \\
& = a^\top w \\
& = \phi(s,a)^\top w \\
& = Q_{w,+}(s,a).
\end{align}
Thus $Q_{w,+}$ is the action value function of the target policy on $M_{w,+}$.

Similarly we verify that $Q_{w,-}$ solves the Bellman evaluation equation on $M_{w,-}$.
\begin{align}
(\T^{\pi_w}_{w,-}Q_{w,-})(s,a)
& = r_{w,-}(s,a) + \gamma Q_{w,-}(\successor_{w}(a),\pi_w(\successor_{w}(a))) \\
& = r_{w,-}(s,a) + \gamma Q_{w,-}(a,\pi_w(a)).
\end{align}
We now evaluate the RHS. Two cases are possible: if $a \not \in \C_\gamma(w)\cup \C_\gamma(-w)$ then the reward function is zero and the rhs of the above equation reads
\begin{align}
\gamma Q_{w,-}(a,\pi_w(a)) & = \gamma \phi(a,\pi_w(a))^\top (-w) \\
& = \gamma \( \frac{1}{\gamma} (a^\top w) w \)^\top (-w) \\
& = -a^\top w \\
& = \phi(s,a)^\top (-w) \\
& = Q_{w,-}(s,a).
\end{align}
If conversely $a \in \C_\gamma(w)\cup \C_\gamma(-w)$ then the rhs reads
\begin{align}
(1-\gamma)a^\top(-w) + \gamma Q_{w,-}(a,\pi_w(a)) & = (1-\gamma)a^\top(-w) +  \gamma \phi(a,\pi_w(a))^\top (-w) \\
& = (1-\gamma)a^\top(-w) + \gamma a^\top (-w) \\
& = -a^\top w \\
& = \phi(s,a)^\top (-w) \\
& = Q_{w,-}(s,a).
\end{align}
\end{proof}

\subsection{Proof of Theorem \ref{lowerbound:thm:OPE:Policy-Induced}}
\begin{proof}
Consider the MDP class described in \cref{sec:OPE:Policy:class} and \cref{sec:OPE:Policy:instance}. First, from \fullref{lowerbound:lem:OPE:Policy:Realizability} we know that every instance of the OPE problem satisfies \fullref{lowerbound:asm:OPE:MDP} with a given feature map $\phi(\cdot,\cdot)$.

Next, we know that the set of query points is induced by one or more policies. However notice that any fixed policy visits the same states regardless of the MDP because the transition function does not depend on the particular MDP at hand. Therefore, the policy-induced query set is also policy-free. Thus, it suffices to do the proof for a policy free query set $\mu$. In addition, the size of the policy-free query set $\mu$ is exactly $n$ since the MDP is deterministic.

From \fullref{lowerbound:lem:Existence} we know that if $|\mu| < I^{-1}_{1-\gamma^2}(\frac{d-1}{2},\frac{1}{2})$ then $\exists \wtilde \in \partial\B^{}$  such that $\forall (s,a) \in \mu,  a \not \in \C_\gamma(\wtilde) \cup  \C_\gamma(-\wtilde)$. Consider the two associated MDPs $M_{\wtilde,+}$ and $M_{\wtilde,-}$.  Notice that the transition function is by construction identical on $M_{\wtilde,+}$ and $M_{\wtilde,-}$ (and so is the target policy) while their reward functions are zero at any $(s,a)\in \mu$:
\begin{align}
\begin{cases}
r_{\wtilde,+}(s,a) & = r_{\wtilde,-}(s,a) \\
p_{\wtilde,+}(s,a) & = p_{\wtilde,-}(s,a).
\end{cases}
\end{align}
This implies that the reward and transition functions in the dataset could have originated from either $M_{\wtilde,+}$ or $M_{\wtilde,-}$, and likewise the target policy does not indicate whether the prediction concerns $M_{\wtilde,+}$ or $M_{\wtilde,-}$.

In $s^\star = \vec 0$ the value of $Q^{\pi_{\wtilde}}(s^\star,\wtilde)$ is $+1$ on $M_{\wtilde,+}$ and $-1$ on $M_{\wtilde,-}$. The agent has thus two choices: either predict a positive value or a negative one. At best, it can randomize between the two, showing the result.
\paragraph{Learning with an online algorithm with knowledge of the target policy}  
Now consider the following online algorithm that submits policy-induced queries. The algorithm first asks the target policy to output an action on the states $\frac{\gamma}{2} e_1,\dots,\frac{\gamma}{2}  e_d$, which is a basis for $\R^d$. This reveals the vector $w$ of the MDP and the position of each hyperspherical cap. Then the agent probes either hyperspherical cap to gain knowledge about the reward function, which identifies the sign of the MDP.

By construction, notice that none of the state vectors $\frac{\gamma}{2} e_1,\dots,\frac{\gamma}{2}  e_d$ is located in an hyperspherical cap of height $\gamma$ since $|(\frac{\gamma}{2} e_i)^\top w| \leq \frac{\gamma}{2} \| e_i \|_2 \| w \|_2 < \gamma$. Since $e_1,\dots,e_d$ is a basis for $\R^d$, we have that $w \neq 0$ cannot be orthogonal to all of them. 
Therefore, there must exist an index $j \in [d]$ such that $e_j^\top w \neq 0$. The action that the target policy outputs $\pi_w(\frac{\gamma}{2} e_j)$ for any such vector that satisfies $\frac{\gamma}{2} e_j^\top w \neq 0$ and is not in $ \C_\gamma(w)\cup \C_\gamma(-w)$  must be non-zero and must be aligned with  $w$. 
Thus, the agent knows the vector $w$ of the MDP (up to a sign). 
Next, it can play $(s^\star,w)$ (a behavioral policy of length $1$ from the initial state $s^\star$) to probe one of the two hyperspherical caps and to observe the reward $\pm (1-\gamma) $ which identifies the sign of the MDP. Since the specific MDP is now precisely identified, the agent can predict the value of any policy, and in particular, the action-value function of the target policy.
\end{proof}

\newpage
\section{Proof of Theorem \ref{lowerbound:thm:BPI:Policy-Induced} and Theorem \ref{lowerbound:thm:BPI:ExponentialSeparation}}
\subsection{MDP Class}
\label{sec:BPI:Policy:class}

We consider a class $\M$ of MDPs sharing the same state-space $\mathcal S$, action space $\ActionSpace$ and discount factor $\gamma > 0$, but different transition function $p$ and reward function $r$. We first describe the state and action space and the discount factor which are fixed across all MDPs in the class $\M$. 

At a high level, each MDP contains a two-armed bandit instance in $s^\star$ (the starting state). There, the learner has two choices: either take the special action $a^\star$ that leads to the terminal state $s^\dagger$ and gives an immediate reward of $\frac{1}{2}$, or take any other action that leads to either an intermediate state $\overline s$ or the terminal state $s^\dagger$. In any case, the agent never gets back to $s^\star$. 

\paragraph{State space} 
The state space can written as the union of three states, a starting state $s^\star$, an intermediate state $\overline s$ and a terminal state $s^\dagger$. Mathematically:
\begin{align}
\StateSpace = \{s^\star,\overline s,s^\dagger\}.
\end{align} 

\paragraph{Action space}
The action space is as follows. In the starting state $s^\star$ the special action $a^\star$ is available in addition to any action in $\B^+$. In $\overline s$, any action in $\B^+$ is available. Finally, in the terminal state $s^\dagger$ only $\vec 0 \in \B^+$ is available. Mathematically:
\begin{align}
\ActionSpace_s \defeq \begin{cases}
\B^+ \cup \{ a^\star \}  & \text{if} \; s = s^\star \\
\B^+ & \text{if} \; s = \overline s  \\
\{\vec 0\}  & \text{if} \; s = s^\dagger.
\end{cases}
\end{align}
Notice that $\vec 0 \in \B^+$.

\paragraph{Discount factor}
The discount factor $\gamma$ is in the interval $(0,1)$.

\paragraph{Feature map}
For this BPI problem the feature map is the $(d+1)$-dimensional vector  defined as follows
\begin{align}
\phi(s,a) = 
\begin{cases}
[a,0] & \text{if}\; a\in \B^+ \\
[\vec 0,1] & \text{if}\; a = a^\star \quad \text{(this action is only available in $s^\star$).}
\end{cases}
\end{align}
We notice that we must have $\| \phi(\cdot,\cdot) \|_2 \leq 1$ and that the feature map only depends on the action.

\subsection{Instance of the Class}
\label{sec:BPI:Policy:instance}

Next we describe the MDP-specific transition and reward functions. Every MDP is identified by a vector $w \in \partial\B^+$ such that $e_i^\top w > 0, \forall i \in [d]$, and by a $+$ and $-$ sign.

\paragraph{Transition Function}
The transition function $p_w$ depends on the vector $w$ that identifies each MDP in the class, but not on the sign $+$ or $-$. Fix the MDP by fixing $w$ (two MDPs correspond to a given choice of $w$). If the agent plays the special action $a^\star$, which is only available in the starting state $s^\star$, it transitions with probability one to the terminal state $s^\dagger$. Otherwise, the transition function is only a function of the action and the successor state is $\overline s$ with some probability, and is otherwise the absorbing state $s^\dagger$. 

Mathematically, if $a = a^\star$ (which implies $s = s^\star$) then
\begin{align*}
p_w(s' = s^\dagger \mid (s^\star,a^\star)) & = 
1 \\
p_w(s' \neq s^\dagger \mid (s^\star,a^\star)) & = 0.
\end{align*} 
This is a valid definition in $(s^\star,a^\star)$. If conversely $a \in \B^+$:
\begin{align*}
p_w(s' = \overline s \mid (s,a)) & = 
\min\{ \frac{1}{\gamma} a^\top w, 1 \} \\
p_w(s' = s^\dagger \mid (s,a)) & = 1 - p_w(s' = \overline s \mid (s,a)).
\numberthis{\label{eqn:Transition_BPI_policy}}
\end{align*} 

Since the probabilities are positive (notice that in particular $a^\top w \geq 0$ since $a\in\B^+$ and $w\in\B^+$) and add up to one, the definition is well posed. In particular, the definition implies that the successor state is always either $\overline s$ or the terminal state $s^\dagger$:
\begin{align}
\forall (s,a): \qquad p_w(s' = s^\star \mid (s,a)) & = 0.
\end{align}

\paragraph{Reward function}
The reward function $r_{w,+}$ or $r_{w,-}$ depends on both the vector $w \in \partial\B^+$ and on the sign $+$ or $-$ that identifies the MDP. It is always $\frac{1}{2}$ if the special action $a^\star$ is taken and otherwise it is everywhere $0$ on $M_{w,-}$ or is positive only in the hyperspherical cap on $M_{w,+}$. Mathematically, it is defined as follows:

On $M_{w,+}$:
\begin{align}
\forall (s,a), \quad r_{w,+}(s,a) \defeq \begin{cases} 
\frac{1}{2} & \text{if} \; (s,a) = (s^\star,a^\star)\\
\max\{a^\top w - \gamma, 0 \} & \text{otherwise.} \;
\end{cases}
\end{align}
On $M_{w,-}$:
\begin{align}
\forall (s,a), \quad r_{w,-}(s,a) \defeq \begin{cases} 
\frac{1}{2} & \text{if} \; (s,a) = (s^\star,a^\star)\\
0 & \text{otherwise.} \;
\end{cases}
\end{align}

\begin{remark}
$M_{w,+}$ and $M_{w,-}$ have identical transition functions but different reward functions. 
\end{remark}
\begin{remark}
The reward functions on $M_{w,+}$ and $M_{w,-}$ differ only when the chosen action is inside the hyperspherical cap $\mathcal C_\gamma(w)$.
\end{remark}

\subsection{Realizability}
We compute the optimal action value function for each MDP in the class $M$, showing realizability.  
\begin{lemma}[$Q^\star$ is Realizable]
\label{lowerbound:lem:BPI:Policy:Realizability}
For any $w\in\partial\B^+$, let $\Qstar_{w,+}$ and $\Qstar_{w,-}$ be the optimal $\Qstar$ values on $M_{w,+}$ and $M_{w,-}$, respectively. Then it holds that
\begin{align}
\label{eqn:BPI:Policy:Solution}
\forall (s,a), \quad \begin{cases}
\Qstar_{w,+}(s,a) = \phi(s,a)^\top [w,\frac{1}{2}] \quad \text{on} \; M_{w,+},\\
\Qstar_{w,-}(s,a) = \phi(s,a)^\top [\vec{0},\frac{1}{2}] \quad \text{on} \; M_{w,-}.
\end{cases}
\end{align}
\end{lemma}
\begin{proof}
We first consider $M_{w,+}$. On any state $\neq s^\dagger$ the optimal policy is to take action $a = w$, achieving a return of $(1-\gamma)$ at every timestep. This way the agent transitions to the state $s = \overline s$ and then stays put there playing action $a = w$. This yields a total return of $1$. In the terminal state $s^\dagger$ action $\vec 0$ is the only available and gives a reward of zero with a self loop. Thus:
\begin{align}
\Vstar_{w,+}(s) = \begin{cases}
1 & \text{if} \; s \neq s^\dagger \\
0 & \text{if} \; s = s^\dagger.
\end{cases}
\end{align}

Now we apply the Bellman operator to $\Vstar_{w,+}$ to compute the optimal action-value function on $M_{w,+}$.

If $(s,a) = (s^\star,a^\star)$ then
\begin{align}
(\T_{w,+} \Vstar_{w,+})(s^\star,a^\star) = r_{w,+}(s^\star,a^\star) + \gamma\E_{s' \sim p_w(s^\star,a^\star)}\Vstar_{w,+}(s') = \frac{1}{2} + 0.
\end{align}
If $(s,a) \neq (s^\star,a^\star)$ and $a  \not \in \C_\gamma(w)$ then by definition $a^\top w < \gamma $ and so
\begin{align}
(\T_{w,+} \Vstar_{w,+})(s,a) = r_{w,+}(s,a) + \gamma\E_{s' \sim p_w(s,a)}\Vstar_{w,+}(s') = 0 + \gamma \( \frac{1}{\gamma}a^\top w\cdot 1 + 0\) = a^\top w.
\end{align}
If $(s,a) \neq (s^\star,a^\star)$ and $a  \in \C_\gamma(w)$ then 
\begin{align}
(\T_{w,+} \Vstar_{w,+})(s,a) = r_{w,+}(s,a) + \gamma\E_{s' \sim p_w(s,a)}\Vstar_{w,+}(s') = (a^\top w - \gamma) + \gamma \cdot 1 = a^\top w.
\end{align}

This shows that the optimal action-value function on $M_{w,+}$ is
\begin{align}
\Qstar_{w,+}(s,a) = \begin{cases}
[a,0]^\top[w,\frac{1}{2}] & \text{if}\; (s,a) \neq (s^\star,a^\star) \\
[\vec 0,1]^\top[w,\frac{1}{2}] & \text{if}\; (s,a) = (s^\star,a^\star)
\end{cases}
\end{align}
which equals $\phi(s,a)^\top[w,\frac{1}{2}]$.

Now we reason on $M_{w,-}$. In $s^\star$, the optimal policy is play $a^\star$ once and transition to the terminal state $s^\dagger$. If conversely $s \neq s^\star$, the maximum attainable return by any policy is $0$. Thus on $M_{w,-}$ the optimal value function reads
\begin{align}
\Vstar_{w,-}(s) = 
\begin{cases}
\frac{1}{2} & \text{if} \; s = s^\star, \\
0 & \text{if} \; s \neq s^\star.
\end{cases}
\end{align}

Now we apply the Bellman optimality operator to $\Vstar_{w,-}$ to obtain the optimal action value function on $M_{w,-}$.

If $(s,a) = (s^\star,a^\star)$
\begin{align}
(\T_{w,-} \Vstar_{w,-})(s^\star,a^\star) = r_{w,-}(s^\star,a^\star) + \gamma\E_{s' \sim p_w(s^\star,a^\star)}\Vstar_{w,-}(s') = \frac{1}{2} + 0.
\end{align}
Conversely, if  $(s,a) \neq (s^\star,a^\star)$
\begin{align}
(\T_{w,-} \Vstar_{w,-})(s,a) = r_{w,-}(s,a) + \gamma\E_{s' \sim p_w(s,a)}\Vstar_{w,-}(s') = 0 = a^\top \vec 0
\end{align}
This shows that on $M_{w,-}$ the optimal action value function is
\begin{align}
\Qstar_{w,-}(s,a) = \begin{cases}
[a,0]^\top[\vec 0,\frac{1}{2}] & \text{if}\; (s,a) \neq (s^\star,a^\star) \\
[\vec 0,1]^\top[\vec 0,\frac{1}{2}] & \text{if}\; (s,a) = (s^\star,a^\star)
\end{cases}
\end{align}
which equals $\phi(s,a)^\top[\vec 0,\frac{1}{2}]$.
This concludes the proof.
\end{proof}

\subsection{Proof of Theorem \ref{lowerbound:thm:BPI:Policy-Induced}}
\begin{proof}
Consider the MDP class described in \cref{sec:BPI:Policy:class} and \cref{sec:BPI:Policy:instance}. First, from \fullref{lowerbound:lem:BPI:Policy:Realizability} we know that every member of the class satisfies \fullref{lowerbound:asm:BPI:MDP} with a given feature map $\phi(\cdot,\cdot)$. 

Notice that any policy $\pi$ induces the same state-actions (possibly with the exception of the singletons $(s^\star,a^\star)$ and $(s^\dagger,\vec 0)$) regardless of the MDP that the learner is interacting with. We can therefore consider the case that the oracle has chosen a policy-free query set $\mu = \{(s,a)\}_{i=1,2,\dots,n}$ of size $n$ in addition to the singletons $(s^\star,a^\star)$ and $(s^\dagger,\vec 0)$ (these singletons do not convey any additional information). Here the $(s,a)$ pairs are the possible state-actions visited by the behavioral policies chosen by the oracle. 

From \fullref{lowerbound:lem:LonelyPlus} we know that if $|\mu| < 2^{-d}I^{-1}_{1-\gamma^2}(\frac{d-1}{2},\frac{1}{2})$ then $\exists \wtilde \in \partial\B^{+}$ such that $e_i^\top w > 0, \forall i \in [d]$ and  $\forall (s,a) \in \mu,  a \not \in \C_\gamma(\wtilde)$. Consider the two associated MDPs $M_{\wtilde,+}$ and $M_{\wtilde,-}$.  Notice that the transition function is by construction identical on $M_{\wtilde,+}$ and $M_{\wtilde,-}$ while their reward functions are zero at any $(s,a)\in \mu$:
\begin{align}
\begin{cases}
r_{\wtilde,+}(s,a) & = r_{\wtilde,-}(s,a) \\ p_{\wtilde,+}(s,a) & = p_{\wtilde,-}(s,a).
\end{cases}
\end{align}
This implies that the transitions and the rewards in the dataset could have originated from either $M_{\wtilde,+}$ or $M_{\wtilde,-}$. Thus in $s^\star$, the batch algorithm has two choices to determine $\pihatstar_{\mathcal D}$: 
choose action $a^\star$ and get a total return of $\frac{1}{2}$ or choose an action $a \neq a^\star$. The second choice  is $\frac{1}{2}$-suboptimal on $M_{\wtilde,-}$, while the first is $\frac{1}{2}$-suboptimal on $M_{\wtilde,+}$. At best, the batch algorithm can randomize between the two, showing the result.
\end{proof}

\subsection{Proof of Theorem \ref{lowerbound:thm:BPI:ExponentialSeparation}}
\begin{proof}
Now consider the following adaptive algorithm that submits policy-induced queries. The algorithm first plays $\gamma e_1,\dots,\gamma e_d$ in $s^\star$ (these are $d$ policies that generate trajectories of length one) to locate the position of the hyperspherical cap (vector $w$). Then the agent probes the hyperspherical cap to gain knowledge of the reward function, which identifies the sign of the MDP.

Upon playing $\gamma e_1,\dots,\gamma e_d$ in $s^\star$ the agent receives the transition functions $p_w(s' = \overline s \mid (s^\star,\gamma e_i))  = 
\min\{ \frac{1}{\gamma} (\gamma e_i)^\top w, 1 \} =  e_i^\top w $ for all $i \in [d]$. From this, the agent can determine each entry of the vector $w$. Next, it can play the state-action $(s^\star,w)$ to probe the hyperspherical cap and to observe the reward ($1-\gamma$ on $M_{w,+}$ and $0$ on $M_{w,-}$) which identifies the sign of the MDP. Since the specific MDP is now precisely identified, the agent can predict the value of any policy, and in particular, it can return the optimal policy.
\end{proof}

\newpage
\section{Proof of Theorem \ref{lowerbound:thm:ALL:Policy-Free}}
\subsection{MDP Class Definition}
\label{sec:ALL:Distribution:class}
We consider a class $\M$ of MDPs sharing the same state-space $\mathcal S$, action space $\ActionSpace$ and discount factor $\gamma > 0$, but different reward function $r$, transition function $p$ and target policy (for the OPE problem). A feature map $\phi$ from the state-action space to $\R^d$ is prescribed  and is fixed across different MDPs of the class. We first define the state and action space, the discount factor and the feature extractor which are fixed across all MDPs in the class $\M$.

We present \emph{two versions} of the construction, one with \emph{continuous} action space and one with \emph{small action} space. The proof for either case is the same; the two constructions are reported here to highlight the impact of the action space in such lower bound. 

\paragraph{State space} 
The state space $\mathcal S$ consists of a starting state  $s^\star$ and  a set of satellite states which can be identified with the unit ball $\B$.  Mathematically we can write
\begin{align}
\StateSpace = \{ s^\star \} \cup \B, \quad \{ s^\star \} \cap \B = \emptyset.
\end{align}

\paragraph{Action space (Continuous $\ActionSpace$)}
The starting state $s^\star$ has actions in the Euclidean ball; each satellite state has a unique action. Mathematically	
\begin{align}
\label{eqn:ALL:Distribution:ActionSet}
\ActionSpace(s) \defeq 	\begin{cases}
\B \quad & \text{if} \quad  s = s^\star \\
\{s\} \quad & \text{if} \quad  s \in \B. 
\end{cases}
\end{align}
In particular $\forall s \in \StateSpace, \ActionSpace(s) \subseteq \B$.

\paragraph{Action space (Small $\ActionSpace$)}
The starting state $s^\star$ has the canonical vectors in the Euclidean ball as available actions (along with their `negative counterpart'); each satellite state has a unique action. Mathematically	
\begin{align}
\label{eqn:ALL:Distribution:ActionSetExtended}
\ActionSpace(s) \defeq 	\begin{cases}
E \defeq \{e_1,\dots,e_d,-e_1,\dots,-e_d \} \quad & \text{if} \quad  s = s^\star \\
\{s\} \quad & \text{if} \quad  s \in \B. 
\end{cases}
\end{align}
In particular $\forall s \in \StateSpace, \ActionSpace(s) \subseteq \B$.

\paragraph{Discount factor}
The discount factor $\gamma$ is in the interval $(0,1)$.

\paragraph{Feature map}
The feature extractor returns the action chosen in the selected state. Mathematically  
\begin{align}
\phi(s,a) & = a.
\end{align}
Since the available actions are always a subset of $\B$ (see e.g., \cref{eqn:ALL:Distribution:ActionSet}), it is easy to see that the image of the feature map is (contained in) the set $\B$.

\paragraph{Target Policy} In the OPE problem, the target policy is identical in every MDP $M$ in the class, and in particular, it returns the only action available in each state $\neq s^\star$. In $s^\star$ it takes action $\vec 0$.

\subsection{Instance of the Class}
\label{sec:ALL:Distribution:instance}
Every MDP in the class $\M$ can be identified by a vector $w \in \B$ and a sign $+$ or $-$, and is denoted by $M_{w,+}$ or $M_{w,-}$, respectively. Next we describe the transition function and the reward function on $M_w$ and the target policy.

\paragraph{Transition function}
The transition function depends only on the vector $w\in \B$ that identifies the MDP $M_{w}$, but not on the sign $+$ or $-$ that distinguishes $M_{w,+}$ from $M_{w,-}$. For a given MDP $M_w$, the transition function is deterministic and depends only on the action chosen; it is convenient to represent the only possible successor state $s' \in \StateSpace$ by the function $s' =\successor_w(a)$.
Mathematically
\begin{align}
\label{eqn:ALL:Distribution:TransitionFunction}
\successor_w(a) = 
\begin{cases}
\frac{1}{\gamma} (a^\top w) w, \quad \quad \quad & \text{if} \; a\not \in \C_\gamma(w)\cup \C_\gamma(-w), \\
a \quad \quad \quad & \text{otherwise}.
\end{cases}
\end{align} 

It is easy to see that the linear algebra operations in the above display are well defined. In addition, $\successor_w(a) \in \B \subset \StateSpace$, because $a \in \B$ and if $a \not \in \C_\gamma(w)\cup \C_\gamma(-w)$ we have $\| \frac{1}{\gamma} (a^\top w) w \|_2 \leq \frac{|a^\top w|}{\gamma}\| w \|_2 \leq 1$, so $\frac{1}{\gamma} (a^\top w) w \in \B$ as well. This also means that the successor state is never the starting  state $s^\star$.

\paragraph{Reward function}
The reward function depends on both the vector $w \in \B$ and on the sign $+$ or $-$ that identifies the MDP and it is defined as follows:
\rdef

\begin{remark}
$M_{w,+}$ and $M_{w,-}$ have identical transition functions but opposite reward functions.
\end{remark}
\begin{remark}
\label{rem:sign}
$M_{w,+}$ and $M_{-w,-}$ have identical transition and reward functions i.e., they are the same MDP $M_{w,+}= M_{-w,-}$. Likewise, we have $M_{w,-}=M_{-w,+}$.
\end{remark}

\subsection{Realizability}
For each MDP in the class, we compute the action value function of an arbitrary policy $\pi$, showing realizability.

\begin{lemma}[Realizability: Verifying assumption \ref{lowerbound:asm:ALL:MDP}]
\label{lowerbound:lem:ALL:Distribution:Realizability}
For any vector $w \in \partial \B$,
and any policy $\pi$, let $Q^\pi_{w,+}$ and $Q^\pi_{w,-}$ be the action-value functions of $\pi$ on $M_{w,+}$ and $M_{w,-}$, respectively. Then it holds that
\begin{align}
\label{eqn:ALL:Distribution:Solution}
\forall (s,a), \quad \begin{cases}
Q^\pi_{w,+}(s,a) = \phi(s,a)^\top (+w) \quad \text{on} \; M_{w,+},\\
Q^\pi_{w,-}(s,a) = \phi(s,a)^\top (-w) \quad \text{on} \; M_{w,-}.
\end{cases}
\end{align}
\end{lemma}
\begin{proof}
Let $\T^{\pi}_{w,+}$ and $\T^{\pi}_{w,-}$ be the Bellman evaluation operators for $\pi$ on $M_{w,+}$ and $M_{w,-}$, respectively.
We need to show that the proposed solutions in \cref{eqn:ALL:Distribution:Solution} satisfy the Bellman evaluation equations at all $(s,a)$ pairs.

First, we focus on $M_{w,+}$ and apply the Bellman operator $\T^\pi_{w,+}$ to $\Qpi_{w,+}$. 

If $a \not \in \C_{\gamma}(w) \cup \C_{\gamma}(-w)$ then notice that the successor state is $s^+(a) = \frac{1}{\gamma} (a^\top w) w$. Furthermore, $\pi$ must  return the only action available in the successor state (and the successor state is never $s^\star$). Thus we can write 
\begin{align}
(\T^{\pi}_{w,+}Q^\pi_{w,+})(s,a) & = r_{w,+}(s,a) + \gamma Q^\pi_{w,+}(\successor_{w}(a),\pi(\successor_{w}(a))) \\
& = 0 + \gamma Q^\pi_{w,+}(\successor_{w}(a),\pi(\successor_{w}(a))) \\
& = \gamma Q^\pi_{w,+}(\successor_{w}(a),\successor_{w}(a)) \\
& = \gamma \phi(\successor_{w}(a),\successor_{w}(a))^\top(+w) \\
& = \gamma (\successor_{w}(a))^\top (+w) \\
& = \gamma \(\frac{1}{\gamma} (a^\top w) w\)^\top (+w) \\
& = a^\top w \\
& = \phi(s,a)^\top w \\
& = Q^\pi_{w,+}(s,a).
\end{align}

If conversely $a \in \C_{\gamma}(w) \cup \C_{\gamma}(-w)$ then $\successor_{w}(a) = a$, and as before the policy can only take the only action available there, giving
\begin{align}
(\T^{\pi}_{w,+}Q^\pi_{w,+})(s,a) & = (1-\gamma)a^\top(+w) + \gamma Q^\pi_{w,+}(\successor_{w}(a),\pi(\successor_{w}(a))) \\
& = (1-\gamma)a^\top(+w) + \gamma Q^\pi_{w,+}(a,\pi(a)) \\
& = (1-\gamma)a^\top(+w) + \gamma \phi(a,a)^\top(+w) \\
& = (1-\gamma)a^\top(+w) + \gamma a^\top (+w) \\
& = a^\top (+w) \\
& = \phi(s,a)^\top (+w) \\
& = Q^\pi_{w,+}(s,a).
\end{align}
This shows that $Q^\pi_{w,+}(s,a) = \phi(s,a)^\top (+w)$ is the value of $\pi$ on $M_{w,+}$.

The argument to verify that $Q^{\pi}_{w,-} = \phi(s,a)^\top (-w)$ is the value of $\pi$ on $M_{w,-}$ is identical, as follows.

If $a \not \in \C_{\gamma}(w) \cup \C_{\gamma}(-w)$ then
\begin{align}
(\T^{\pi}_{w,-}Q^\pi_{w,-})(s,a) & = r_{w,-}(s,a) + \gamma Q^\pi_{w,-}(\successor_{w}(a),\pi(\successor_{w}(a))) \\
& = 0 + \gamma Q^\pi_{w,-}(\successor_{w}(a),\pi(\successor_{w}(a))) \\
& = \gamma Q^\pi_{w,-}(\successor_{w}(a),\successor_{w}(a)) \\
& = \gamma (\successor_{w}(a))^\top (-w) \\
& = \gamma \(\frac{1}{\gamma} (a^\top w) w\)^\top (-w) \\
& = a^\top (-w) \\
& = \phi(s,a)^\top (-w).
\end{align}

Otherwise, if $a \in \C_{\gamma}(w) \cup \C_{\gamma}(-w)$ then\begin{align}
	(\T^{\pi}_{w,-}Q^\pi_{w,-})(s,a) & = (1-\gamma)a^\top(-w) + \gamma Q^\pi_{w,-}(\successor_{w}(a),\pi(\successor_{w}(a))) \\
	& = (1-\gamma)a^\top(-w) + \gamma Q^\pi_{w,-}(a,\pi(a)) \\
& = (1-\gamma)a^\top(-w) + \gamma a^\top (-w) \\
& = a^\top (-w) \\
& = \phi(s,a)^\top (-w).
\end{align}
\end{proof}

\subsection{Proof of Theorem \ref{lowerbound:thm:ALL:Policy-Free}}

\begin{proof}
Consider the MDP class described in \cref{sec:ALL:Distribution:class} and \cref{sec:ALL:Distribution:instance}. First, from \fullref{lowerbound:lem:ALL:Distribution:Realizability} we know that every MDP in $\M$ satisfies \fullref{lowerbound:asm:ALL:MDP} with a given feature map $\phi(\cdot,\cdot)$.

Let $\mu$ be the policy-free query set chosen by the oracle. Using \fullref{lowerbound:lem:Existence} we can claim that if $|\mu| < I^{-1}_{1-\gamma^2}(\frac{d-1}{2},\frac{1}{2})$ then $\exists \wtilde \in \partial\B^{}$ such that $\forall (s,a) \in \mu,  a \not \in \C_\gamma(\wtilde) \cup  \C_\gamma(-\wtilde)$. Consider the two associated MDPs $M_{\wtilde,+}$ and $M_{\wtilde,-}$.  Notice that the transition function is by construction identical on $M_{\wtilde,+}$ and $M_{\wtilde,-}$ while their reward function is zero at any $(s,a)\in \mu$:
\begin{align}
\begin{cases}
r_{\wtilde,+}(s,a) & = r_{\wtilde,-}(s,a) \\
p_{\wtilde,+}(s,a) & = p_{\wtilde,-}(s,a).
\end{cases}
\end{align}
This implies that the reward and transition functions in the dataset could have originated from either $M_{\wtilde,+}$ or $M_{\wtilde,-}$. In addition, the target policy is fixed for all MDPs and so it does not reveal the MDP instance.

\paragraph{Conclusion for continuous action space}

In this section we focus on the construction with continuous actions space.

Consider the best policy identification problem. In $s^\star$ the algorithm has two choices to determine $\pihatstar_{\mathcal D}$: choose an action $a$ such that $a^\top \wtilde > 0$ or $a^\top \wtilde \leq 0$. In the first case, it obtains negative return on $M_{\wtilde,-}$ and in the second case it obtains negative return on $M_{\wtilde,+}$. However, the value of the optimal policy in $s^\star$ is $+1$ in both cases. Even if the batch algorithm randomizes the output, with probability at least $1/2$ the returned policy is at least $1$-suboptimal.

Likewise, consider the off-policy evaluation problem where $\pi$ is the target policy. The batch algorithm has two choices to estimate $Q^{\pi}_M(s^\star,\wtilde)$: either a positive or a negative value. However $Q^{\pi}_{M_{\wtilde,+}}(s^\star,\wtilde)$ equals $+1$ and $Q^{\pi}_{M_{\wtilde,-}}(s^\star,\wtilde)$ equals $-1$. At best, the batch algorithm  can randomize between a positive and a negative value, making an error of at least $1$ with probability at least $1/2$.

\paragraph{Conclusion for small action space}

In this section we focus on the construction with small actions space.

Now \fullref{lowerbound:lem:SufficientInnerProduct} ensures that that for any $\wtilde, \|\wtilde \|_2 = 1$ there exists an action $e_j \in \ActionSpace(s^\star)$ such that either $e_j^\top \wtilde \geq \frac{1}{\sqrt{d}}$ or $-e_j^\top \wtilde \geq \frac{1}{\sqrt{d}}$; define the action in $s^\star$ most aligned with $\wtilde$ to be $\widetilde e = \argmax_{e \in \ActionSpace(s^\star)} e^\top \wtilde \geq \frac{1}{\sqrt{d}}$. 

Consider the best policy identification problem. In $s^\star$ the algorithm has two choices to determine $\pihatstar_{\mathcal D}$: choose an action $a$ such that $a^\top \wtilde > 0$ or $a^\top \wtilde \leq 0$. In the first case, it obtains negative return on $M_{\wtilde,-}$ and in the second case it obtains negative return on $M_{\wtilde,+}$. However, the value of the optimal policy in $s^\star$ is at least $\widetilde e^\top \wtilde \geq \frac{1}{\sqrt{d}} $ on $M_{w,+}$ and at least  $(-\widetilde e)^\top (-\wtilde) \geq \frac{1}{\sqrt{d}} $ on $M_{w,-}$ (notice that $\widetilde e$ is available in the construction with small action space). Even if the batch algorithm randomizes the output, with probability at least $1/2$ the returned policy is at least $\frac{1}{\sqrt{d}}$-suboptimal.

Likewise, consider the off-policy evaluation problem where $\pi$ is the target policy. The batch algorithm has two choices to estimate $Q^{\pi}_M(s^\star,\widetilde e)$: either a positive or a negative value. However $Q^{\pi}_{M_{\wtilde,+}}(s^\star,\widetilde e) \geq \frac{1}{\sqrt{d}}$ and $Q^{\pi}_{M_{\wtilde,-}}(s^\star,\widetilde e) \leq -\frac{1}{\sqrt{d}}$. At best, the batch algorithm  can randomize between a positive and a negative value, making an error of at least $\frac{1}{\sqrt{d}}$ with probability at least $1/2$.

\paragraph{Learning with an adaptive or online algorithm (continuous action space)}
Now consider the following adaptive algorithm that submits policy-free queries (the algorithm that submits policy-induced queries is analogous). At a high level, the algorithm first plays $\frac{\gamma}{2} e_1,\dots,\frac{\gamma}{2}  e_d$ in $s^\star$. This is a basis for $\R^d$ and by construction it does not intersect any hyperspherical cap. However, this  allows the agent to locate the position of the hyperspherical caps in the given MDP (vector $w$). Then the agent probes either hyperspherical cap to gain knowledge of the reward function, which identifies the sign of the MDP.

More formally, consider an MDP $M_w, w \in \partial\B$ and notice that none of the query actions (described above) is located in an hyperspherical cap of height $\gamma$:  we have $|(\frac{\gamma}{2} e_i)^\top w| \leq \frac{\gamma}{2} \| e_i \|_2 \| w \|_2 < \gamma$. In addition, $e_1,\dots,e_d$ is a basis for $\R^d$ which implies  that $w \neq 0$ cannot be orthogonal to all of them. Therefore, there must exist an index $j \in [d]$ such that $e_j^\top w \neq 0$. Since $\frac{\gamma}{2}e_j \not \in \C_{\gamma}(w) \cup \C_{\gamma}(-w)$, the successor state $s^+_w(\frac{\gamma}{2} e_j)$ must be non-zero and must be parallel to $w$. Thus, the agent knows the vector $w$ of the MDP (up to a sign). Next, it can play $(s^\star,w)$ to probe one of the two hyperspherical caps and to observe the reward $\pm (1-\gamma) $ which identifies the sign of the MDP using \cref{rem:sign}. Since the specific MDP is now precisely identified, the agent can predict the value of any policy, and therefore, it can return an optimal policy.
\end{proof} 

\subsection{Helper Lemma}

To prove \fullref{lowerbound:thm:ALL:Policy-Free} with small action space we need the following helper lemma.
\begin{lemma}[Sufficient Inner Product]
\label{lowerbound:lem:SufficientInnerProduct}
Let $\{e_1,\dots,e_d\} \subseteq \R^d$ be the set of canonical vectors and $\wtilde \in \R^d, \| \wtilde \|_2 = 1$. Then there exists a canonical vector $e_j$ such that
\begin{align}
	|e_j^\top \wtilde | \geq \frac{1}{\sqrt{d}}.
\end{align}
\begin{proof}
Let $\wtilde_j$ be the $j$ th component of $\wtilde$. From the hypothesis we must have
\begin{align}
\label{eqn:sum}
	1 = \| \wtilde \|^2_2 = \sum_{j=1}^d \wtilde^2_j.
\end{align}	
This implies that at least one component $\wtilde_j$ for some $j \in [d]$ must be greater (in absolute value) than $\frac{1}{\sqrt{d}}$, i.e., 
\begin{align}
	\exists j \in [d] \quad \text{such that} \quad |\wtilde_j| \geq \frac{1}{\sqrt{d}} 
\end{align}
otherwise $\| \wtilde \|_2 < 1$ in \cref{eqn:sum}, contradiction. Since $e_1,\dots,e_j$ are the canonical vectors, the thesis $| e_j^\top \wtilde | = | \wtilde_j | \geq \frac{1}{\sqrt{d}}$ now follows.
\end{proof}

\end{lemma}